\documentclass{article}

\usepackage{multirow}
\usepackage{microtype}
\usepackage{graphicx}
\usepackage{subfigure}
\usepackage{booktabs} 
\usepackage{cite}
\usepackage{lipsum}
\usepackage{amsthm}
\usepackage{times}
\usepackage{epsfig}
\usepackage{graphicx}
\usepackage{microtype}
\usepackage{graphicx}
\usepackage{subfigure}
\usepackage{booktabs} 
\usepackage{amsmath}
\usepackage{amssymb}
\usepackage{diagbox}
\usepackage{url}
\usepackage{caption}
\usepackage{wrapfig}
\usepackage[super]{nth}
\usepackage{physics}
\usepackage{pifont}
\usepackage{amsmath}
\usepackage{relsize}
\usepackage{adjustbox}
\usepackage{enumitem}
\usepackage{algorithm}
\usepackage[noend]{algorithmic}
\usepackage{graphicx}
\usepackage{epstopdf}

\newtheorem{prop}{Proposition}

\newtheorem{defn}{Definition}

\usepackage{mathtools}
\usepackage{amsbsy} 

\setlist{nosep}

\DeclarePairedDelimiter\floor{\lfloor}{\rfloor}

\DeclareMathOperator*{\argmin}{arg\,min}
\usepackage{xparse}

\ExplSyntaxOn
\NewDocumentCommand{\stack}{O{c}m}
 {
  \begin{tabular}[#1]{@{}c@{}}
  \tl_map_function:nN { #2 } \__tom_stack:n
  \end{tabular}
 }
\cs_new_protected:Nn \__tom_stack:n { #1 \\ }
\ExplSyntaxOff
\usepackage{tikz}
\usepackage{soul}
\makeatletter
\def\SOUL@soeverytoken{%
 {\the\SOUL@token}\par\noindent}
\makeatother
\newcommand{\cmark}{\ding{51}}%
\newcommand{\xmark}{\ding{55}}%

\usepackage{hyperref}

\usepackage[accepted]{icml2021}

\icmltitlerunning{Spatial Transformer $\boldsymbol{K}$-Means}

\begin{document}

\twocolumn[
\icmltitle{Spatial Transformer $\boldsymbol{K}$-Means}



\icmlsetsymbol{equal}{*}

\begin{icmlauthorlist}
\icmlauthor{Romain Cosentino}{to}
\icmlauthor{Randall Balestriero}{to}
\icmlauthor{Yanis Bahroun}{goo}
\icmlauthor{Anirvan Sengupta}{goo,ed}
\icmlauthor{Richard Baraniuk}{to}
\icmlauthor{Behnaam Aazhang}{to}
\end{icmlauthorlist}

\icmlaffiliation{to}{Rice University}
\icmlaffiliation{goo}{Flatiron Institute}
\icmlaffiliation{ed}{Rutgers University}

\icmlcorrespondingauthor{Romain Cosentino}{rom.cosentino@gmail.com}

\icmlkeywords{Machine Learning, ICML}

\vskip 0.3in
]



\printAffiliationsAndNotice 

\begin{abstract}
$K$-means defines one of the most employed centroid-based clustering algorithms with performances tied to the data's embedding. Intricate data embeddings have been designed to push $K$-means performances at the cost of reduced theoretical guarantees and interpretability of the results. Instead, we propose preserving the intrinsic data space and augment $K$-means with a similarity measure invariant to non-rigid transformations. This enables (i) the reduction of intrinsic nuisances associated with the data, reducing the complexity of the clustering task and increasing performances and producing state-of-the-art results, (ii) clustering in the input space of the data, leading to a fully interpretable clustering algorithm, and (iii) the benefit of convergence guarantees.
\end{abstract}

\section{Introduction}

Clustering algorithms aim at discovering patterns in the data that enable their characterization, identification, and separation. The development of such a framework without any prior information regarding the data remains one of the milestones of machine learning that would assist clinicians, physicists, and data scientists, among others, with a better pattern discovery tool \cite{bertsimas2020interpretable,greene2005producing}.

While supervised learning has been converging toward the almost exclusive use of Deep Neural Networks (DNN), avoiding the development of handcrafted features to provide the desired linearly separable embedding map, unsupervised clustering algorithms take various forms depending on the application at hand \cite{ma2019learning,wagstaff2001constrained,estivill2002so}. For instance, the usage of SIFT features combined with clustering algorithm for medical imaging \cite{nam2009high}, the extraction of DNNs embedding used as the input of the $K$-means algorithm for computer vision tasks \cite{xie2016unsupervised}, and the combination of signal-processing features extractors combined with Gaussian mixture model to understand the nature of the various seismic activities \cite{seydoux2020clustering}. The important role of clustering algorithms in assisting medical diagnoses as well as scientific discoveries highlight the importance of the development of an \textit{interpretable} and \textit{theoretically guaranteed} tool \cite{dolnicar2003using,xu2010clustering}.

In this work, we focus our attention on the $K$-means clustering algorithm \cite{macqueen1967some} and its application to $2$-dimensional signals, such as images or time-frequency representations. Well-known for its simplicity, efficiency, and interpretability, the $K$-means algorithm partitions the data space into $K$ disjoint regions. Each region is represented by a centroid, and each datum is assigned to the closest centroid's region. 
The integral part in the design of a clustering algorithm is the choice of an appropriate distance, and the number of clusters \cite{he2013k,frey2002fast,990920}. While the Euclidean distance makes the design of the algorithm straightforward, this measure of similarity might omit the geometrical relationships between data points \cite{steinbach2004challenges}. In fact, a small rigid perturbation of an image, such as rotation or translation, is enough to change the cluster assignment.  

There are two major difficulties in constructing a distance for a clustering algorithm; on the one hand, the metric should take into account the geometry of the data, e.g., be invariant to rigid transformations for images, and on the other hand, the metric should be interpretable as it is tied to the interpretability of the algorithm \cite{steinbach2004challenges}.


In this work, we tackle these two difficulties by introducing in our similarity measure the spatial transformations inherent to the geometry of the data at hand.  
In particular, we: $(i)$ formulate an interpretable and theoretically guaranteed $K$-means framework capable of exploiting the symmetry within the data, $(ii)$ extend prior work on metrics invariant to rigid transformations to non-rigid transformations, thus taking into account a more realistic set of nuisances and $(iii)$ allow the learnability of the symmetry underlying the data at hand, therefore enabling the exploration of data where the equivalence classes are yet to be determined. 

To learn the symmetry in the data and perform their transformations, we will use the spatial transformer framework, which was successfully introduced in \citet{jaderberg2015spatial}. This allows us to provide a learnable metric invariant to non-rigid transformations that is used as the $K$-means distortion error.

While many approaches to learn and estimate non-rigid transformations have been proposed, we will follow one of nowadays mainstream approaches developed in \citet{jaderberg2015spatial} where the Thin Plate Spline is used as a differentiable deformation model. Our attempt is, in fact, not to compare among deformation models but to consider a way to approach the learnability of invariances in an unsupervised setting such that it is effective, tractable, and interpretable.

Our contributions can be summarized as follows:
\begin{itemize}[leftmargin=*, itemsep=0pt]
    \item We propose a novel approach to tackle clustering using a novel adaptive similarity measure within the $K$-means framework that considers non-rigid transformations, Sec.~\ref{sec:similarity}.
    \item We derive an appropriate update rule for the centroids that drastically improves both the interpretability of the centroids and their quality, Sec.~\ref{sec:learning}.
    \item We provide convergence guarantees and geometrical interpretations of our approach, Sec.~\ref{sec:convergence},~\ref{sec:geometry}.
    \item Finally, we show numerically that our unsupervised algorithm competes with state-of-the-art methods on various datasets while benefiting from interpretable results, Sec.~\ref{sec:res}.
\end{itemize}

\section{Background}
\label{sec:background}

\subsection{Invariant Metrics}

The development of measures invariant to specific deformations has been under investigation in the computer vision community for decades \citep{fitzgibbon2002affine,simard2012transformation,lim2004image}. By considering affine transformations such as shearing, translation, and rotation of the data as being nuisances, these approaches propose a distance that reduces the variability intrinsic to high-dimensional images. These works are considered as appearance manifold-based framework; that is, the distance are quantified by taking into account geometric proximity \cite{262951,basri1998clustering,su2001modified,1211332}. 

While the development of affine invariant metrics is pretty standard, their extension to more general non-rigid transformations requires more attention. Recently, various deep learning methods proposed ways to learn diffeomorphic transformations \cite{detlefsen2018deep,balakrishnan2018unsupervised,lohit2019temporal,dalca2019learning,NEURIPS2019_db98dc0d}. Others adopt a more theoretically grounded approach based on group theory as in \citet{zhang2015finite,freifeld2015highly,durrleman2013sparse,allassonniere2015bayesian} as well as the statistical ``pattern theory'' approach developed in \citet{grenander1993general,dupuis1998variational}.

\subsection{Spatial Transformer}

\begin{figure}[t]
   \centering
   \begin{minipage}{.32\linewidth}
   \includegraphics[width=\textwidth]{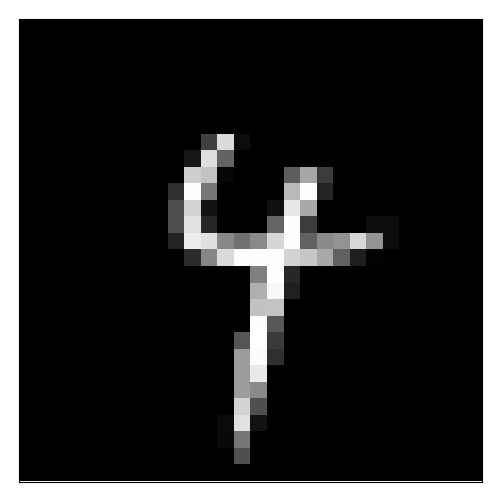}
   \end{minipage}
     \begin{minipage}{.32\linewidth}
   \includegraphics[width=\textwidth]{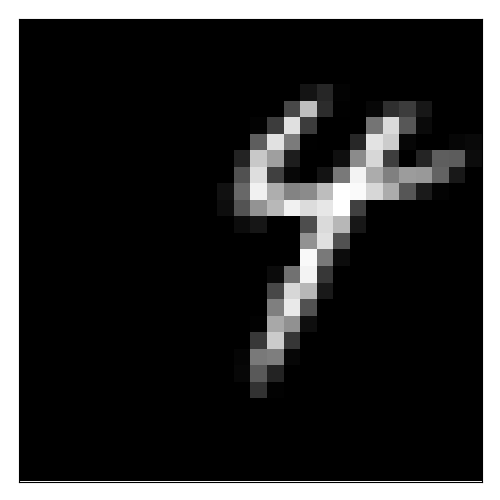}
   \end{minipage}
         \begin{minipage}{.32\linewidth}
   \includegraphics[width=\textwidth]{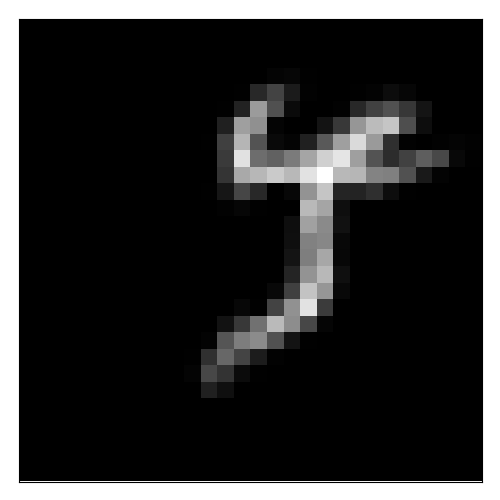}
   \end{minipage}
         \begin{minipage}{.32\linewidth}
   \includegraphics[width=\textwidth]{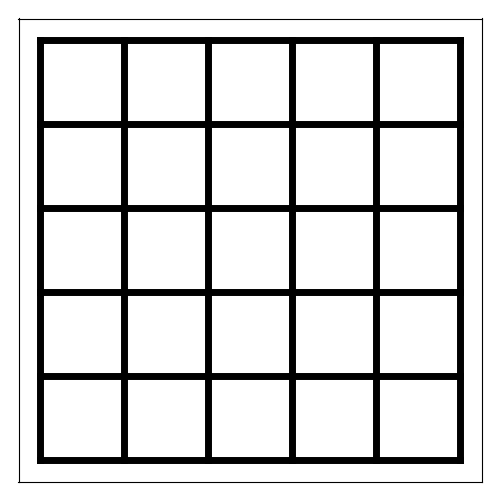}
   \end{minipage}
      \begin{minipage}{.32\linewidth}
   \includegraphics[width=\textwidth]{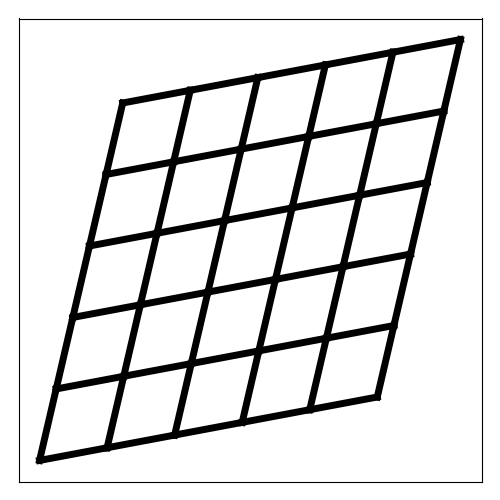}
   \end{minipage}       
      \begin{minipage}{.32\linewidth}
   \includegraphics[width=\textwidth]{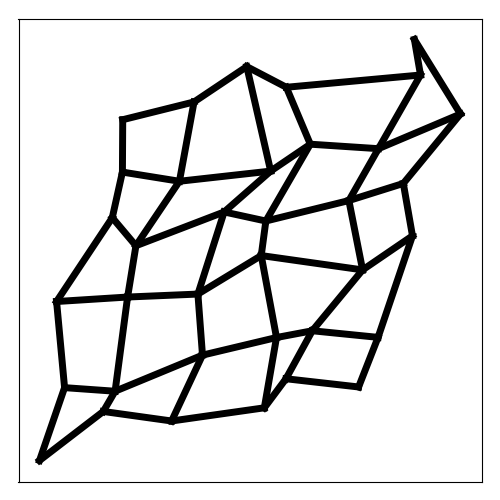}
   \end{minipage}   
   \caption{\textbf{Spatial Transformations - } Visualizations of a sample taken from the MNIST dataset and its transformed versions. Each image results from the application of the spatial transformer that take as input the original signal (top left), and the grid displayed below its transformed version. (\textit{Left}) we observe the original image and its associated original transformation grid, which corresponds to the identity transform. (\textit{Middle)} the image has been transformed by the affine transformation induced by the associated grid. (\textit{Right}) the image transformed by the non-rigid transformation using the TPS induced by the grid below it.} 
   \label{fig:orig_to_diffeo}
 \end{figure}
The transformer operator, denoted by $\mathcal{T}$, allows for non-rigid image transformations. It is based on the composition of two mappings; a deformation map and a sampling function. The deformation maps a uniform grid of $2$-dimensionak coordinates to provide its transformed version. The sampling function samples the signal with respect to a given grid of $2$-dimensional coordinates.

The mapping we select to enable the learnability of the transformation in the coordinate space of the $2$-dimensional signal is the Thin-Plate-Spline (TPS) interpolation technique \cite{duchon1976interpolation,bookstein1989principal,nejati2010fast} which produces smooth surfaces from $\mathbb{R}^2$ to $\mathbb{R}^2$ \cite{morse2005interpolating}. We refer the reader to Appendix~\ref{app:TPS} for details regarding this method.
We consider as learnable parameters of the TPS a set of $2$-dimensional coordinates, called landmarks, and denoted by $\nu$. Given a set of landmarks, the TPS provides the transformation map of a $2$-dimensional grid. That is, the euclidean plane is bent according to the learned landmarks.

In Fig.~\ref{fig:orig_to_diffeo}, we show on the bottom right the grid associated with the $\ell=6^2$ landmarks. Each grid corresponds to the spatial transformation applied to the hand-written digit $4$. The transformation of the signal based on these new coordinates is produced by performing bilinear interpolation using the original signal (top left) and the new coordinates; the details are provided in Appendix~\ref{app:TPS}.

The spatial transformer is the composition of these two maps and is defined as
\begin{equation}
    \mathcal{T}(x,\nu),
    \label{eq:tps_notation}
\end{equation}
where $x \in \mathbb{R}^n$ is the original $2$-dimensional signal, $\nu \in \mathbb{R}^{2\ell}$ is the set of $2$-dimensional transformed coordinate to be learned. Note that $2\ell$ can be smaller than the dimension of the image as the TPS interpolates to re-scale the transformation to any size.

Such a framework composing the TPS and bilinear interpolation has been defined as spatial transformer in \citet{jaderberg2015spatial}. However, in their work, the inference of the non-rigid transformations is performed using each datum as the input of a ``localisation network"; instead, we directly learn the transformation parameters. 

\section{Spatial Transformer $K$-means}
\label{sec:RAI-Kmeans}
We now introduce the spatial transformer $K$-means, ST $K$-means, our proposed solution that composes the spatial transformer and the $K$-means algorithm.

\subsection{Formalism}
\label{sec:similarity}
We recall that in this work we will consider $2$-dimensional signals defined by their width and height, such as images and time-frequency representation of time-series. Given a set of $2$-dimensional signals ,$\{x_i \}_{i=1}^{N}$, with $x_i \in \mathbb{R}^{n}$, the $K$-means algorithm aims at grouping the data into $K$ distinct clusters defining the partition $\mathcal{C} = \left \{ C_k \right \}_{k=1}^{K}$, with $\cup_k C_k=\{x_i \}_{i=1}^{N}$ and $C_i\cap C_j = \emptyset, \forall i\neq j$. Each cluster $C_k$ of the partition is represented by a centroid $\mu_k \in \mathbb{R}^n, \forall k \in \{1,\dots, K\}$.

As for the $K$-means algorithm, the goal of the ST $K$-means is to find centroids minimizing the following distortion error
\begin{equation}
\begin{aligned}
& \min_{\mathcal{C}, \mu_{1},\ldots, \mu_K}
& & \sum_{k=1}^{K} \sum_{i: x_i \in C_k}   d(x_i, \mu_k)~~.
\label{eq:RAI_Kmeans}
\end{aligned}
\end{equation}
The assignment of a signal $x_i$ to a cluster $C_k$ is achieved through the evaluation of the similarity measure, $d$, between the signal and each centroid. A signal $x_i$ belongs to the cluster $C_l$ if and only if $l = \argmin_k d(x_i, \mu_k)$. While the standard $K$-means algorithm makes use of the Euclidean distance, i.e., $d(x_i, \mu_k)= \|x_i - \mu_k \|_2^2$, we instead propose to use the following deformation invariant similarity measure
\begin{align}
    &d(x_i,\mu_k) := \min_{\nu \in \mathbb{R}^{2\ell}} \left \| \mathcal{T}(x_i,\nu) - \mu_k \right \|_2^2~~,  \label{eq:def_invariant}  
\end{align}
which is a \textit{Quasipseudosemimetric}, see Appendix~\ref{app:proofs} for details and proof.

This similarity measure represents the least-square distance between the centroids and the datum that has been fit to the centroid via the spatial transformer operator. Once this fitting is done for each centroid, the cluster assignment is done based on the argmin of those distances, i.e.,  the data $x_i$ is assigned to $\argmin_{k} d(x_i,\mu_k)$. Therefore, the underlying assumption of our approach is that the distance between the optimal transformation of a signal into a centroid belonging to the same "class" should be smaller than the distance between its optimal transformation into a centroid that does not. That is, let $x_i$ be geometrically near $\mu_k$, then $\min_{\nu \in \mathbb{R}^{2\ell}} \left \| \mathcal{T}(x_i,\nu) - \mu_k \right \|_2^2 < \min_{\nu \in \mathbb{R}^{2\ell}} \left \| \mathcal{T}(x_i,\nu) - \mu_k' \right \|_2^2$.

This measure requires solving a non-convex optimization problem. It can be achieved in practice by exploiting the spatial transformer's differentiability with respect to the landmarks $\nu$. As a result, we can learn the transformation by performing gradient-descent based optimization \cite{kingma2014adam}; further details regarding this optimization are given in Appendix~\ref{app:impl_details} as well as solutions to facilitate the optimization of the non-convex objective by exploiting the manifold geometry.

The crucial property of the measure we propose is its invariance to deformations that are spanned by the spatial transformer; formal proofs and definitions are proposed in Appendix~\ref{app:invariance}. 
This means that evaluating Eq.~\ref{eq:def_invariant} with any datum that is transformed from the spatial transformer will produce the same value, as long as no information is lost.

\subsection{Learning the Spatial Transformer $K$-means}
\label{sec:learning}
Solving the optimization problem in Eq.~\ref{eq:RAI_Kmeans}, similarly to $K$-means, is an NP-hard problem. A popular tractable solution nonetheless exists and is known as the two-step Lloyd algorithm \cite{lloyd1982least}. 

In the ST $K$-means, the first step of the Lloyd algorithm consists of assigning the data to the clusters using the newly defined measure of similarity in Eq.~\ref{eq:def_invariant}~. The second step is the update of the centroids using the previously determined cluster assignment. It corresponds to the result of the optimization problem: $\argmin_{\mu_k} \sum_{i: x_i \in C_k} d(x_i,\mu_k)$, provided in following Proposition~\ref{prop2}.
\begin{prop}
\label{prop2}
The centroids update of the ST $K$-means algorithm are given by
\begin{equation}
\normalfont   \mu_k^{\star} := \frac{1}{\left | C_k \right |}  \sum_{i: x_i \in C_k} \mathcal{T}(x_i,\nu_{i,k}^{\star}),  \: \: \forall k 
   \label{eq:update}
\end{equation}
where $\left | C_k \right |$ denotes the cardinal of the set $C_k$, $\normalfont \nu_{i,k}^{\star}$ is the set of parameters of the TPS that best transforms the signal $x_i$ into the centroid $\mu_k$, that is, $\nu_{i,k}^{\star} = \argmin_{\nu \in \mathbb{R}^{2l}} \left \| \mathcal{T}(x_i,\nu) - \mu_k \right \|_2^2$ (proof in Appendix~\ref{proof:prop2}).
\end{prop}

The averaging in Eq.~\ref{eq:update} is performed on the transformed version of the signals. The ST $K$-means thus considers the topology of the signal's space. A pseudo-code of the centroid update Eq.~\ref{eq:update} is presented in Algo.~\ref{algo:update}. 
 \begin{algorithm}
 \caption{Centroids Updates of ST $K$-means}
 \begin{algorithmic}[1]
 \renewcommand{\algorithmicrequire}{\textbf{Input:} }
 \renewcommand{\algorithmicensure}{\textbf{Output:} }
 \REQUIRE  Cluster  $C_k$,  TPS parameters $\left \{ \nu_{i,k}^{\star} \right \}_{i: x_i \in C_k}$
 \ENSURE  Centroids update $ \mu_{k}^{\star}$
 \STATE Initialize $\mu_k = 0$
 \FOR{$i: x_i \in C_k$}  
  \STATE Compute $\mu_k = \mu_k + \mathcal{T}_{\ell}(x_i;\nu_{i,k}^{\star})$, \; Eq.~\ref{eq:update}
 \ENDFOR
 \STATE $\mu_k^{\star} = \frac{\mu_k}{| C_k |}$
 \end{algorithmic} 
 \label{algo:update}
 \end{algorithm}

The ST $K$-means, which aims to minimize the distortion error Eq.~\ref{eq:RAI_Kmeans} is done by alternating between the two steps detailed above until convergence, as summarized in Algo.~\ref{algo:algo}. 
%
%
%
%
 \begin{algorithm}
 \caption{Spatial Transformer $K$-means}
 \begin{algorithmic}[1]
 \renewcommand{\algorithmicrequire}{\textbf{Input:}  }
 \renewcommand{\algorithmicensure}{\textbf{Output:} }
 \REQUIRE Initial centroids $\mu_k$, dataset $\left \{x_i \right \}_{i=1}^N$
 \ENSURE  Cluster partition  $\left \{C_k \right \}_{k=1}^{K}$
 \REPEAT
 \FOR{$i=1$ to $N$}  
  \FOR{$k=1$ to $K$}
  \STATE Compute and store $d(x_i,\mu_k)$ by solving Eq.~\ref{eq:def_invariant}
  \ENDFOR
  \STATE Assign $x_i$ to $C_l$ where $l = \argmin_{k} d(x_i,\mu_k)$
  \ENDFOR
 \STATE Update the centroid $\mu_k$ using Algo.~\ref{algo:update}
 \UNTIL{Convergence}
 \end{algorithmic} 
 \label{algo:algo}
 \end{algorithm}

The update in Eq.~\ref{eq:update}, induced by our similarity measure, alleviates a fundamental limitation of the standard $K$-means.
In fact, in the standard $K$-means, the average of the data belonging to a cluster $C_k$, $\frac{1}{ | C_k|}\sum_{i: x_i \in  C_k} x_i$, consists of an averaging of the signals without deforming them, which, as a result, does not account for the non-euclidean geometry of the signals \cite{klassen2004analysis,srivastava2005statistical}.

\subsection{Convergence of the Spatial Transformer $K$-means}
\label{sec:convergence}

As we mentioned, our development is motivated by the interest in proposing a novel way to think about invariance in an unsupervised fashion while conserving the interpretability and theoretical guarantees of the $K$-means algorithm. We propose here to prove the convergence of the ST $K$-means algorithm following the generalization of clustering algorithms via the Bregman divergence as developed in \citet{banerjee2005clustering}. In their work, they provide the class of distortion function that admits an iterative relocation scheme where a global objective function, such as the one in Eq.~\ref{eq:RAI_Kmeans}, is progressively decreased. We, therefore, prove that Algo.~\ref{algo:algo} monotonically decreases the distortion error of the ST $K$-means in Eq.~\ref{eq:RAI_Kmeans} which in turn implies that  Algo.~\ref{algo:algo} converges to a local optimal.

\begin{prop}
\label{prop:convergence}
Under the assumption that the spatial transformation optimization problem in Eq.~\ref{eq:def_invariant}, reaches a unique global minimum, the ST $K$-means algorithm described in Algo.~\ref{algo:algo} terminates in a finite number of step at a partition that is locally optimal (Proof in Appendix~\ref{app:proofs}).
\end{prop}

\subsection{Geometrical Interpretation of the Similarity Measure}
\label{sec:geometry}

One of the great benefit of the $K$-means algorithm is the interpretability of the regions composing its partitioning. In particular, they are related to Voronoi diagrams which are well studied partitioning techniques \cite{aurenhammer1991voronoi,aurenhammer2013voronoi}.
Following this framework, we propose now to highlight the regions defined by the ST $K$-means algorithm. This is achieved by analysing the following sets $\forall k \in \left \{1,\dots,K \right \}$
\begin{equation}
    R_{k} = \left \{ x \in \mathbb{R}^n | d(x,\mu_k) \leq d(x,\mu_j),~~ \forall j \neq k \right \},
\end{equation}
where we recall $d(x,\mu_k) =   \min_{\nu \in \mathbb{R}^{2 \ell}} \left \| \mathcal{T}(x,\nu) - \mu_k \right \|_2^2$. Such a partitioning falls in the framework of a special type of Voronoi diagram.
\begin{prop}
The partitioning induced by the ST $K$-means corresponds to a weighted Voronoi diagram where each region's size depends on the per data spatial transformations (proof and details in Appendix~\ref{proof:voronoi}).
\end{prop}
While the Euclidean $K$-means induces a Voronoi diagram where each region is a polytope, the ST $K$-means does not impose such a constraint of its geometry. The similarity measure we propose adapts the geometry of each data to each centroid and thus induces a specific metric space for each data-centroid pair. In particular, for each data-centroid pair, the ST $K$-means has a particular metric that induces the boundary of the regions. In a more general setting, each region is defined as the orbit of the centroid with respect to the transformations induced by the spatial transformer, thus defining regions that depend on the orbit's shape instead of polytopal ones.

This geometric observation can lead to efficient initializations for the ST $K$-means \cite{arthur2006k}, as well as the evaluation of its optimality \cite{bhattacharya2016tight}. Besides, one can perform in depth study to understand the shape of the regions spanned by our approach to understand the fail cases of the algorithm for a particular application \citet{har2014complexity,xia2018unified}. One can also compare the partitioning achieved in our approach with the one of DNN as in \citet{NEURIPS2019_0801b20e} to gain more insights into both models.

\subsection{Computational Complexity \& Parameters}
The time complexity of ST $K$-means is $O(NK ( \ell^3 + \ell n))$. In fact, the ST $K$-means computes a TPS of computational complexity $O(\ell^3 + \ell n)$ for each sample of the $N$ samples and each of the $K$ centroids, as in Eq.~\ref{eq:def_invariant}. In practice, $\ell$ is of the order $2^{6}$.
The number of parameters of the model is $ 2\ell  \times N \times K$; it depends on the number of samples, clusters, and landmarks.  

%
%
To speed up the computation, we $(i)$ pre-compute the matrix inverse responsible for the dominating cubic term, see Appendix~\ref{app:TPS} for implementation details regarding the TPS, and $(ii)$ implement ST $K$-means on GPU with SymJAX \cite{balestriero2020symjax} where high parallelization renders the practical computation time near constant with respect to the number of landmarks as we depict in Fig.~\ref{fig:computational_time}.

\begin{figure}[!h]
\begin{minipage}{.02\linewidth}
\rotatebox{90}{Time (min.)}
\end{minipage}
\begin{minipage}{.47\columnwidth}
    \centering
    \includegraphics[width=\linewidth]{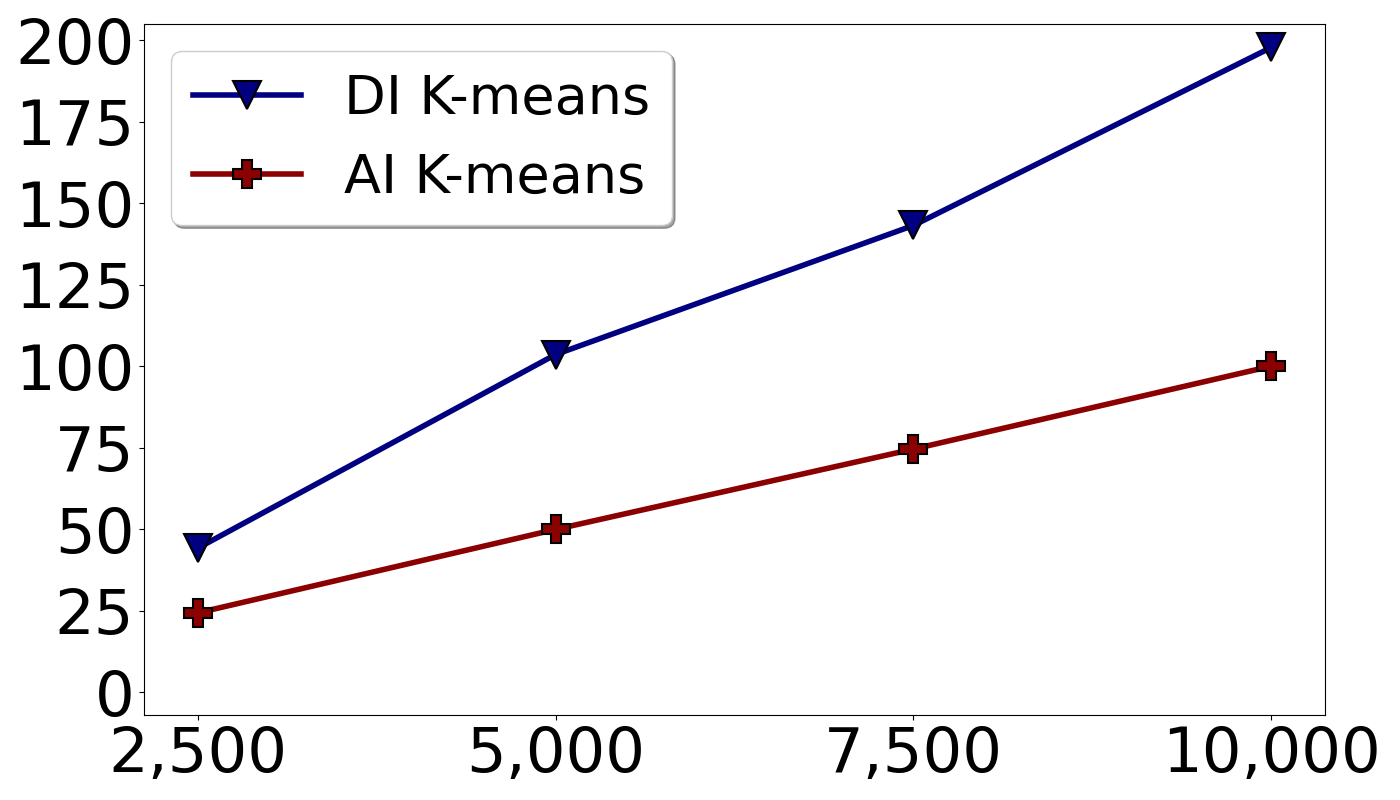}
\end{minipage}
\begin{minipage}{.47\columnwidth}
    \centering
    \includegraphics[width=\linewidth]{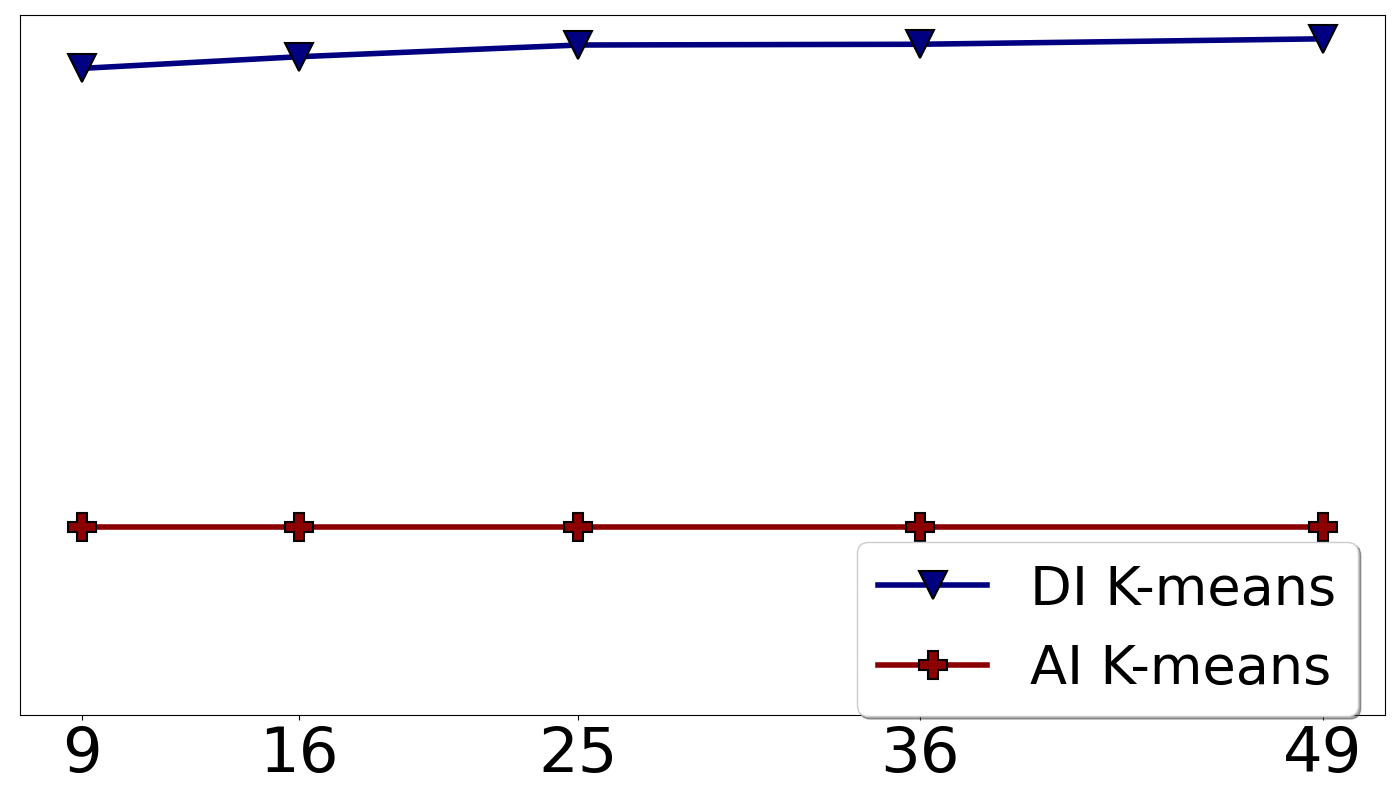}
\end{minipage}
\begin{minipage}{\columnwidth}
\hspace{1.1cm} Training Size ($N$) \hspace{.9cm} Number Landmarks ($\ell$)
\end{minipage}
\caption{ \textbf{Computational Training Time} - Comparison between our ST $K$-means and the Affine Invariant (AI) $K$-means computational times on the Arabic Characters dataset. The input pixel size is $n=1024$. (\textit{Left}) shows the computational time for varying training set sizes and $\ell=7^2$. (\textit{Right}) shows the computational time as a function of the number of landmarks, $\ell$, for $N=10,000$. Since the AI $K$-means does not use the TPS algorithm, its computational time is constant as a function of the number of landmarks. We can observe that our process to speed up the computation enables the tractability of the ST $K$-means.}
\label{fig:computational_time}
\end{figure}

\begin{figure*}[t]
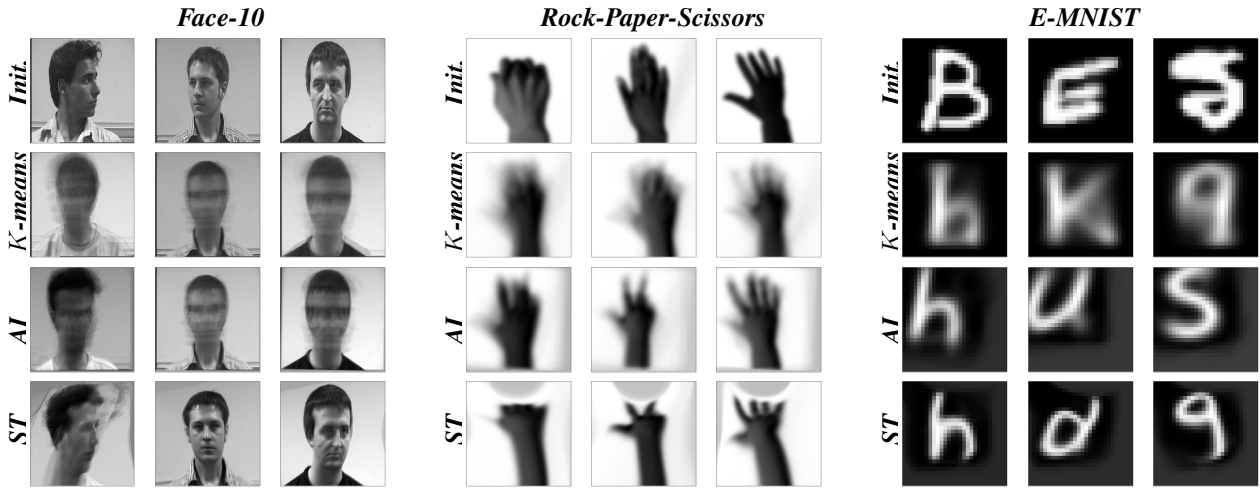


\begin{minipage}[t]{.33\textwidth}
    \centering
    \textbf{\textit{Face-10}}
\end{minipage}
\begin{minipage}[t]{.33\textwidth}
        \centering
    \textbf{\textit{Rock-Paper-Scissors}}
\end{minipage}
\begin{minipage}[t]{.33\textwidth}
        \centering
    \textbf{\textit{E-MNIST}}
\end{minipage}

\begin{minipage}{.008\linewidth}
\rotatebox{90}{ \hspace{.1cm} \textbf{\textit{ST \hspace{.75cm} AI \hspace{.65cm} $K$-means \hspace{.44cm} Init.}}}
\end{minipage}
\begin{minipage}{0.31\linewidth}
\foreach \c in {0,1,7}{
    \begin{minipage}{0.28\linewidth}
    \includegraphics[width=\linewidth]{images/best_facepos_centroid_init\c.png}\\
    \includegraphics[width=\linewidth]{images/kmean_facepos_centroid_final\c.png}\\
    \includegraphics[width=\linewidth]{images/aff_facepos_centroid_final\c.png}\\
    \includegraphics[width=\linewidth]{images/best_facepos_centroid_final\c.png}
    \end{minipage}
}
\end{minipage}
\hfill \hfill
\begin{minipage}{.008\linewidth}
\rotatebox{90}{ \hspace{.1cm} \textbf{\textit{ST \hspace{.75cm} AI \hspace{.65cm} $K$-means \hspace{.44cm} Init.}}}
\end{minipage}
\begin{minipage}{0.31\linewidth}
\foreach \c in {0,...,2}{
    \begin{minipage}{0.28\linewidth}
    \includegraphics[width=\linewidth]{images/best_rockpaper_centroid_init\c.png}\\
    \includegraphics[width=\linewidth]{images/kmean_rockpaper_centroid_final\c.png}\\
    \includegraphics[width=\linewidth]{images/aff_rockpaper_centroid_final\c.png}\\
    \includegraphics[width=\linewidth]{images/best_rockpaper_centroid_final\c.png}
    \end{minipage}
}
\end{minipage}
\hfill \hfill
\begin{minipage}{.008\linewidth}
\rotatebox{90}{ \hspace{.1cm} \textbf{\textit{ST \hspace{.75cm} AI \hspace{.65cm} $K$-means \hspace{.44cm} Init.}}}
\end{minipage}
\begin{minipage}{0.31\linewidth}
\foreach \c in {0,2,3}{
    \begin{minipage}{0.28\linewidth}
    \includegraphics[width=\linewidth]{images/best_emnist_centroid_init\c.png}\\
    \includegraphics[width=\linewidth]{images/kmean_emnist_centroid_final\c.png}\\
    \includegraphics[width=\linewidth]{images/aff_emnist_centroid_final\c.png}\\
    \includegraphics[width=\linewidth]{images/best_emnist_centroid_final\c.png}
    \end{minipage}
}
\end{minipage}

\caption{ \textbf{Centroids} - We depict some centroids for the different $K$-means algorithms. The centroid at initialization are displayed in the \textit{nth{1}} row. The centroids learned by $K$-means are shown in the \textit{\nth{2} row}, by the Affine invariant $K$-means in the \textit{\nth{3} row}, and by our ST $K$-means in the \textit{\nth{4} row}. 
By comparing the results of the AI $K$-means (\textit{\nth{3} row}) with the standard $K$-means (\textit{\nth{2} row}), we can see that using only affine transformations slightly improves the $K$-means centroids and reduces the superposition issue that $K$-means suffers from. 
By comparing the results of our ST $K$-means (\textit{\nth{4} row}) with the other methods, it is clear that using non-rigid transformations significantly improves the quality of the centroids, making them sharper and removing the issue related to the non-additiveness of images. Note that $K$-means iteratively updates the centroids and cluster assignments, as such, the class associated to a specific centroid usually changes during training (additional centroid vizualisations are proposed in Appendix~\ref{app:sup_centroids}).
}
    \label{fig:centroids}
\end{figure*}

\begin{figure*}[p!]
\begin{minipage}{.32\linewidth}
    \centering
\textbf{\textit{Raw Data}}
\end{minipage}
\begin{minipage}{.32\linewidth}
    \centering
\textbf{\textit{Affine Invariant}}
\end{minipage}
\begin{minipage}{.32\linewidth}
    \centering
\textbf{\textit{Spatial Transformer}}
\end{minipage}
\begin{minipage}{0.03\linewidth}
\rotatebox{90}{\hspace{.6cm} \textbf{MNIST \#10} \hspace{2.cm} \textbf{Audio-MNIST \#10} \hspace{1.1cm} \textbf{Face-10 \#13} \hspace{1.5cm} \textbf{Rock-Paper-Scissors \#3}}
\end{minipage}
\begin{minipage}{0.95\linewidth}
    \centering
    \includegraphics[width=\linewidth]{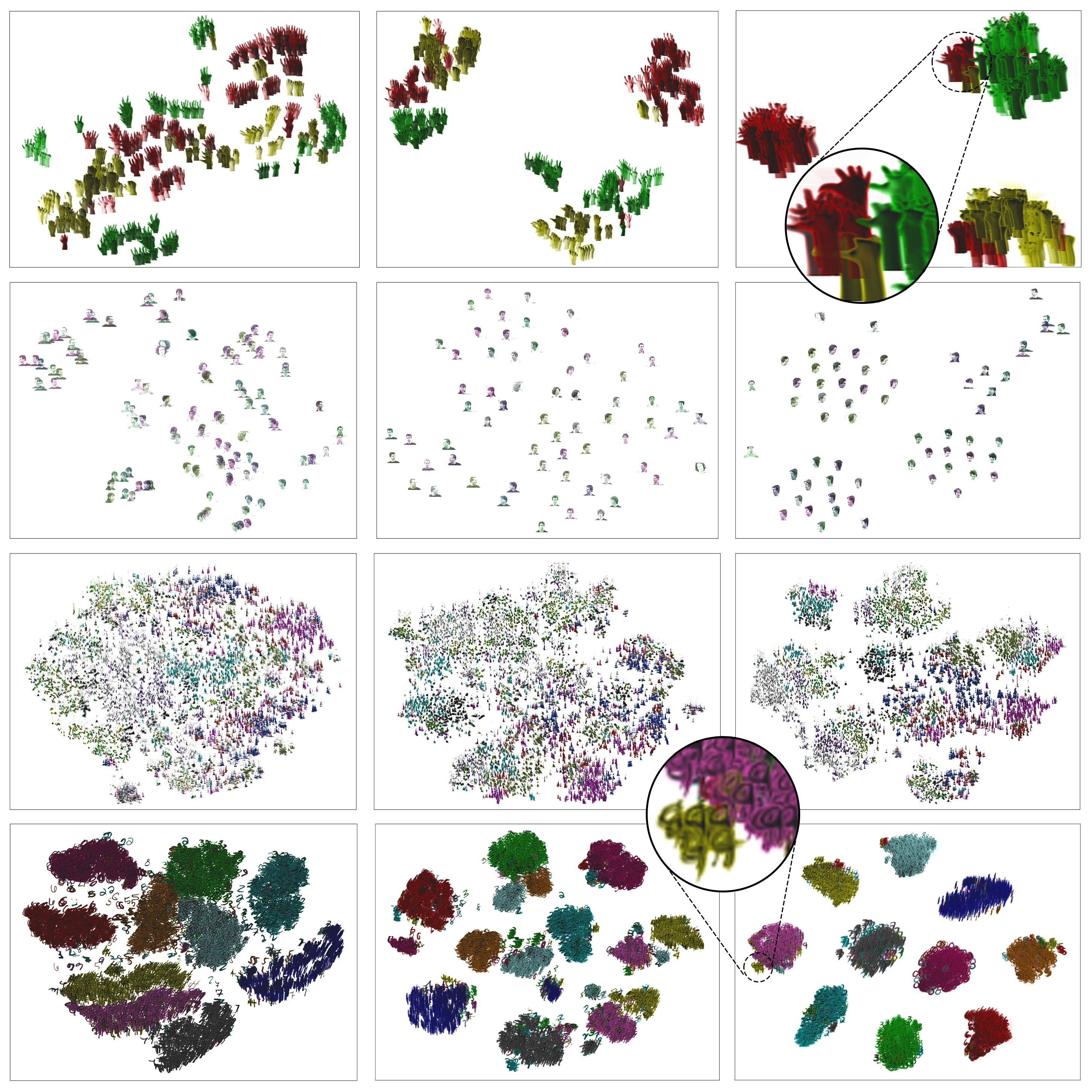}
\end{minipage}
    \caption{\textbf{$2$-dimensional t-SNE} - (\# denotes the number of clusters) - We suggest the reader to zoom in the plots to best appreciate the visualizations. - The raw data (\textit{left column}), the affinely transformed data using the AI distance, i.e., we extract the best affine transformation of the data that corresponds to the centroid it was assigned and perform the t-SNE on these affinely transformed data, (\textit{middle column}), the data transformed with respect to the TPS as per Eq.~\ref{eq:def_invariant}, i.e., the same process as previously mentioned but we consider the spatial transformer instead, and then perform the dimension reduction on these transformed data, (\textit{right column}). Each row corresponds to a different datasets: Rock-Paper-Scissors, Face-10, AudioMNIST, and MNIST  are depicted from the top to bottom row. For all the figures, the colors of the data represent their ground truth labels. We observe that across datasets, both the affine transformations learned on the data and the non-rigid transformations help to define more localized clusters. One can observe that for the Face-10 dataset, while the dataset contains $13$ clusters, we can see that the ST $K$-means induced transformations lead to a $2$-dimensional space where the faces are clustered $3$ majors orientations. The top left cluster corresponds to faces pointing left, the bottom one face pointing right, and the bottom right one face pointing front. We also propose to zoom-in two locations where the ambiguity in the transformation induced by the spatial transformer is noticeable. In particular, we show two cases where the non-rigid transformations are too large for certain samples leading to an erroneous clustering assignment, e.g., in the MNIST dataset, the yellow samples in the lense are initial instance of the class $4$ that have been transformed into digit that geometrically ressemble the centroid of the cluster $9$, thus being assigned to the $9$'s cluster. The same concept is shown in the Rock-Paper-Scissors lense where some instance of the classes rock and paper are assigned to the class scissors (additional t-SNE vizualisations are proposed in Appendix~\ref{app:sup_tsne}).}
\label{fig:tsne}
\end{figure*}

\section{Experimental Setup}
\label{sec:experiments}

%
In this section, we detail the experimental settings followed to evaluate the performances of our model. 
For all the experiments, the number of clusters is set to be the number of classes the dataset contains for all clustering algorithms. 
The various datasets and their train-test split to optimize the model's parameters and update the centroids of the different models are described in Appendix~\ref{ap:data}. 
%

\subsection{Evaluation Metrics}

For all the experiments, the accuracy is calculated using the metric proposed in \citet{yang2010image} and defined as
\setlength{\abovedisplayskip}{3pt}
\setlength{\belowdisplayskip}{3pt}
\begin{align}
 \text{Accuracy} =    \max_{m} \tfrac{1}{N} \sum_{i=1}^{N} 1_{ \left \{ l_i = m(\hat{l}_i) \right \}}~~,
\label{eq:cluster} 
\end{align}
where $l_i$ is the ground-truth label, $\hat{l}_i$ the cluster assignment and $m$ all the possible one-to-one mappings between clusters and labels. 
The results in Table~\ref{table:compare} are taken as the best score on the test set based on the ground truth labels among $10$ runs as in \citet{xie2016unsupervised}. We also provide on the same run the normalized mutual information (NMI) \cite{romano2014standardized}, and adjusted rand index (ARI) \cite{hubert1985comparing}.

\subsection{Cross Validation Settings}

We provide in Appendix.~\ref{app:compete} the details regarding the benchmark models and their cross-validation settings.

Our model requires the cross-validation of hyper-parameters:  the number of landmarks and the learning rate to learn the similarity measure in Eq.~\ref{eq:def_invariant}. 
However, the clustering framework does not allow the use of label information to perform the cross-validation of the parameters. We thus need to find a proxy for it to determine the optimal model parameters. 
Interestingly, the distortion error related used in the ST $K$-means, Eq.~\ref{eq:RAI_Kmeans}, appears to be negatively correlated to the accuracy, as displayed in Fig.~\ref{fig:accu_vs_dist}. 
Note that the use of the distortion error is commonly used as a fitness measure in $K$-means, for example, when cross-validating the number of clusters. 
\begin{figure}[!h]
\begin{minipage}{.04\columnwidth}
\rotatebox{90}{\hspace{.45cm} Accuracy}
\end{minipage}
\begin{minipage}{.96\columnwidth}
    \centering
    \includegraphics[trim=0 0.5cm 0 0,width=1\columnwidth]{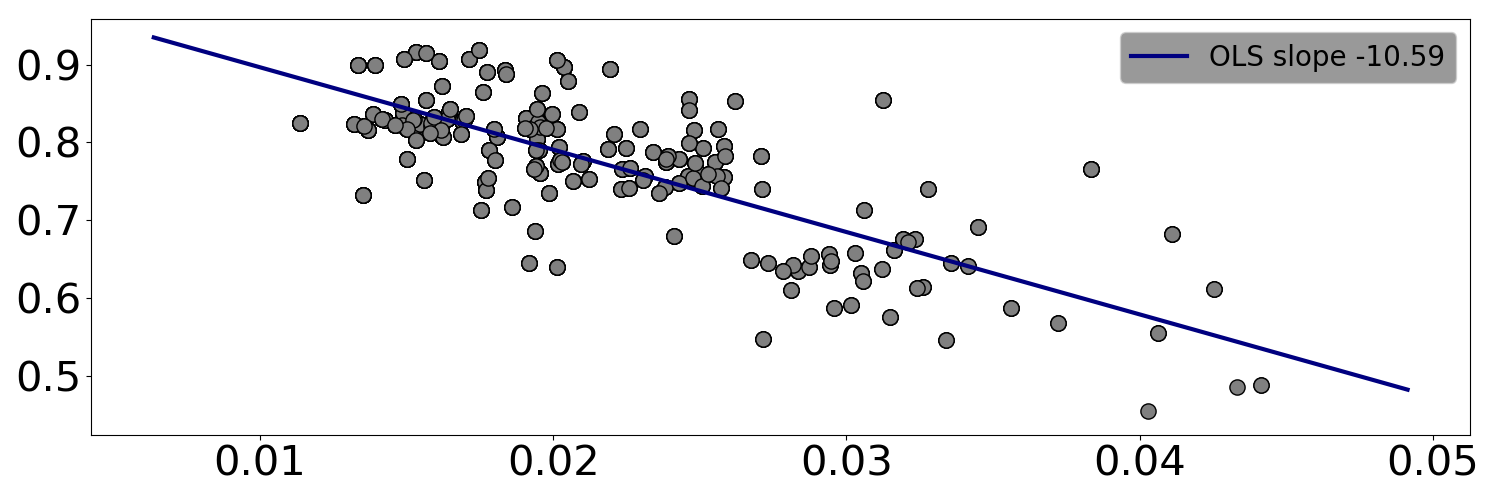}
\begin{minipage}{\columnwidth}
\centering
\hspace{.5cm} Distortion Error
\end{minipage}
\end{minipage}
\vspace{-.25cm}
\caption{\textbf{Accuracy vs Distortion Error} -  
Clustering accuracy, Eq.~\ref{eq:cluster}, of ST $K$-means algorithm on the MNIST dataset as a function of the distortion error, Eq.~\ref{eq:RAI_Kmeans}, using the similarity measure, Eq.~\ref{eq:def_invariant}. Each gray dot is associated with a specific set of hyper-parameters, e.g., the learning rate and the number of landmarks for the spatial transformer. 
%
The accuracy is negatively correlated to the distortion error (see the blue line corresponding to the ordinary least square fit), indicating that the distortion error is an appropriate metric to cross-validate the hyper-parameters of the ST $K$-means algorithm, which is crucial in an unsupervised setting as the labels are not available. }
\label{fig:accu_vs_dist}
\end{figure}

%

We cross-validate the number of landmarks, $\ell$, which defines the resolution of the transformation, which we optimize over the following grid, $\left [3^2,4^2,5^2,6^2,7^2,8^2 \right ]$. 
Then, the learning of the landmarks, $\nu$, is done via Adam optimizer.
The learning rate is picked according to $\left [10^{-4},5\times 10^{-4},10^{-3},5\times 10^{-3},10^{-2},5\times 10^{-2} \right ]$. 
We train our method for $150$ epochs for all the datasets, with batches of size $64$. 
As for $K$-means and AI $K$-means, the centroids' initialization of the ST $K$-means is performed by the $K$-means$++$ algorithm. 
Importantly, the same procedure is applied to all datasets. 

Note that during the training, both the similarity measure in Eq.~\ref{eq:def_invariant} and the clustering update are performed, Eq.~\ref{eq:cluster}. During the algorithm's testing phase, the centroids remain fixed, and only the similarity measure is performed to assign each testing datum to a cluster.

\section{Results and Interpretations}
\label{sec:res}
In this section, we report and interpret the results obtained by our ST $K$-means and competing models. 
\subsection{Clustering Accuracy}
We report in Table~\ref{table:compare} the accuracy of the different models considered on the different datasets. 
Our approach shows to outperform existing models on most datasets. Our model equals the performance of AI $K$-means on Affine MNIST and is only outperformed by VaDE (MLP) on MNIST.

\begin{table*}[t]
\centering
\caption{Clustering results in $\%$ of the test set accuracy Eq.~\ref{eq:cluster} - Following the benchmarks evaluation method, the best accuracy (ACC) over $10$ runs are displayed - We also provide the associated normalized mutual information (NMI) and adjusted rand index (ARI) - the number of clusters is denoted by \# next to the dataset name and where ($\dagger$): \citet{xie2016unsupervised} and ($\ddagger$): \citet{jiang2016variational}.}
 \setlength\tabcolsep{3.pt}
\begin{tabular}{|l|c|c|c|c|c|c|c|c|c||c|c|c|c|c|c|c|c||c|c|c|c|c|c|c|c|}
\cline{2-26}
\multicolumn{1}{c|}{}&\rotatebox{90}{\textit{Deep Learning}}&
\rotatebox{90}{\textbf{Aff. MNIST} \#10}& \rotatebox{90}{\textbf{Diffeo. MNIST} \#10} & \rotatebox{90}{\textbf{MNIST} \#10}& \rotatebox{90}{\textbf{Audio MNIST} \#10} &\rotatebox{90}{\textbf{E-MNIST} \#26} &\rotatebox{90}{\textbf{Rock-Paper-Sci.} \#3} &\rotatebox{90}{\textbf{Face-10} \#13} & \rotatebox{90}{\textbf{Arabic Char.} \#28}
& \rotatebox{90}{\textbf{Aff. MNIST} \#10}& \rotatebox{90}{\textbf{Diffeo. MNIST} \#10} & \rotatebox{90}{\textbf{MNIST} \#10}& \rotatebox{90}{\textbf{Audio MNIST} \#10} &\rotatebox{90}{\textbf{E-MNIST} \#26} &\rotatebox{90}{\textbf{Rock-Paper-Sci.} \#3} &\rotatebox{90}{\textbf{Face-10} \#13} & \rotatebox{90}{\textbf{Arabic Char.} \#28}
&
\rotatebox{90}{\textbf{Aff. MNIST} \#10}& \rotatebox{90}{\textbf{Diffeo. MNIST} \#10} & \rotatebox{90}{\textbf{MNIST} \#10}& \rotatebox{90}{\textbf{Audio MNIST} \#10} &\rotatebox{90}{\textbf{E-MNIST} \#26} &\rotatebox{90}{\textbf{Rock-Paper-Sci.} \#3} &\rotatebox{90}{\textbf{Face-10} \#13} & \rotatebox{90}{\textbf{Arabic Char.} \#28}
\\ \cline{3-26}
\multicolumn{1}{c|}{}& &\multicolumn{8}{c||}{ACC}&\multicolumn{8}{c||}{NMI}&\multicolumn{8}{c|}{ARI } \\ \toprule
$K$-means & \xmark  &   68 & 61 & 53 & 10 & 39 & 40 & 20 & 19  &    - &-   & 50 & 1 & 39 & 5 & 18 & 27 &   - & -  & 39 & 0 & 21 & 4 & 0 & 1 \\\hline
AI $K$-means & \xmark &  \textbf{100}  &  91 & 75 & 29 & 48& 72  & 31 & 30 & - & -  & 62 & 18 & 45 & 30 & 30 & 37 &  -  & -  & 54 & 10 & 26 & 24  & 3 & 17 \\ \hline
\textbf{ST $K$-means} &  \xmark  &  \textbf{100} & \textbf{99} & 92 & \textbf{41} & \textbf{65}  &  \textbf{86} &  \textbf{45} & \textbf{51} &  - & -  & 82 & \textbf{26} & \textbf{63}  & \textbf{63} & \textbf{53}  & \textbf{61} &    - & - & 83 & \textbf{15}  & \textbf{46} & \textbf{63} & \textbf{20} & \textbf{38} \\ \midrule
AE + $K$-means   &   \cmark  & 72  & 60 & 66 & 13 & 41 & 48 & 37 & 23 &   - & - & 64 & 1 & 40 & 9 & 27 & 33 &   -& - & 59  & 0 & 28  & 6  & 26 & 15\\  \hline
DEC (MLP) ($\dagger$)  &  \cmark &  84  & 77 & 84 & 10 & 55 & 46 & 33 & 24&  -  & - &  83 & 1  & 51 & 12 & 20 &  32 &  -  & - & 80 & 0 &31  & 8 & 3 & 13\\\hline
DEC (Conv)  &  \cmark &  70  & 68 & 78  & 15 & 60  & 54 & 38 &  29 &    -& - &  74& 3 & 56 & 18 &  31 & 39 &   - & - &  69 & 1 & 37 & 13 & 17 & 16 \\\hline
VaDE (MLP) ($\ddagger$) &  \cmark &  68  & 65 & \textbf{94} & 11 & 20 & 50 & 36 & 26 &   - & - & \textbf{89} & 1 & 12 & 16 & 27   & 30 & - & - & \textbf{85} & 0  & 8 & 11 & 14 & 10 \\\hline
VaDE (Conv) &  \cmark &  65  & 59 & 81 & 14 &  58 & 55  & 40 & 46 &   - & - & 78 & 2 & 55 &20  & 35 & 53  &  - & - & 80  & 0  & 38  &  15 & 18 & 29 \\\bottomrule
\end{tabular}
\label{table:compare}
\end{table*}

%
%
Whereas the various deep learning approaches perform well on datasets for which their architectures were developed, e.g., MNIST and its derivatives: E-MNIST, Arabic Characters, they show limited performance on higher resolution datasets with a small number of samples, such as Rock-Paper-Scissors, Face-10 as well as the two toy examples. In fact, they are composed of only $700$ training data and $300$ testing data. In the following sections, we interpret various visualizations of the $K$-means variants used in this work. 

\subsection{Interpretability: Centroids Visualization }

We propose in Fig.~\ref{fig:centroids} to visualize the centroids obtained via $K$-means, AI $K$-means, and our ST $K$-means. Supplementary visualizations are provided in Appendix~\ref{app:sup_centroids}.
For each dataset, the first row shows the clusters after initialization from $K$-means$++$.
The three following rows show the centroids obtained via the $K$-means, AI $K$-means, and ST $K$-means algorithms, respectively. 

We observe that, for all datasets, the $K$-means centroids are not lying on the data manifold as they are unrealistic images that could not occur naturally in the dataset. Besides, they appear to be blurry and hardly interpretable. These drawbacks are due to the update rule that consists in the average of the data belonging to each cluster in the pixel space.
%
%
The AI $K$-means algorithm drastically reduces the centroids' blurriness induced by such an averaging as it considers the average of affinely transformed data. 
However, our ST $K$-means produces the crispest centroids and does not introduce any ambiguity in between the different clusters. 
In fact, the update of our method, Eq.~\ref{eq:update}, takes into account the non-linear structure of the manifold by taking the average over data transformed using a non-rigid transformation.

Interestingly, Fig.~\ref{fig:centroids} shows that even if at initialization multiple centroids assigned to the same class are attributed to different clusters, the ST $K$-means is able to recover this poor initialization thanks to its explicit manifold modeling and centroid averaging technique.
%
%
For instance, in the Rock-Paper-Scissors dataset, although at initialization, two centroids correspond to the class paper, the ST $K$-means learns centroids of each of the three classes within this dataset. 
%
%
In the Face-10 dataset, some centroids learned correspond to the rotation of the initialization; even in such extreme change of pose, the centroids remain crisp in most cases. 

\subsection{Interpretability: Embedding Visualization}

To get further insights into the disentangling capability of the ST $K$-means, we compare the $2$-dimensional projections of the data using t-SNE \cite{maaten2008visualizing}, of the $K$-means, AI $K$-means and ST $K$-means. Supplementary visualizations are provided in Appendix~\ref{app:sup_tsne}.

The t-SNE visualizations, for both the AI and ST $K$-means, are obtained by extracting the optimal transformation that led to the assignment. 
Precisely, for each image $x_i$, we compute $l = \argmin_k d(x_i, \mu_k)$  and extract the optimal parameter $\nu_{i,l}^{\star}$ which is then used to obtain the transformed image fed as the input of the t-SNE. 
%
%

\vspace{-.25cm}
We can observe in Fig.~\ref{fig:tsne} that the affine transformations ease the data separation in this $2$-dimensional space. The ST $K$-means also drastically enhances the separability of the different clusters. When using ST $K$-means, the data are clustered based on macroscopically meaningful and interpretable parameters, making the model's performance possible to understand. 
For instance, for the Face-10 dataset, the t-SNE representation of the ST $K$-means clusters' shows that faces are grouped according to three significant orientations, left, right, and front. These three clusters are more easily observed in our ST $K$-means than in the affine invariant model. However, the $13$ different orientations present in the dataset remain too subtle to be captured by the ST $K$-means. 
%
%
%
%
%

For the MNIST dataset, the last row and column of Fig.~\ref{fig:tsne}, we observe that most of the incorrectly clustered images are almost indistinguishable from samples of the cluster they have been attributed. 
In particular, we highlight this by proposing to zoom-in into the cluster of hand-written $9$ in Fig.~\ref{fig:tsne}. We can see that the yellow instances are samples from the class $4$ that have been transformed such that they resemble the $9$'s centroid in Fig.~\ref{fig:centroids}. 
%
%
%
%
%
%
%
We also provide a zoom-in on one of the clusters obtained on the rock-paper-scissors dataset, first row and last column of Fig.~\ref{fig:tsne}. The incorrectly clustered data are the ones that, when transformed, easily fit the scissors shape. 


\vspace{-.25cm}
\section{Conclusion}

Designing an unsupervised algorithm that is robust to non-rigid transformations remains challenging, despite the tremendous breakthrough in machine learning. The problem lies in appropriately limiting the size of the transformations. We showed that the spatial transformer could achieve this as the number of landmarks allows the learnability of a coarse to fine grid of transformation. However, such a parameter controlling the size of the transformation should be designed as well as be learned per-cluster or per-sample. Besides this difficulty, we showed that we could conserve the interpretability of the $K$-means algorithm applied in the input data space while drastically improving its performances. Such a framework should be favored in clustering applications where the explainability of the decision is critical.

\bibliography{example_paper}
\bibliographystyle{icml2021}

\appendix
\clearpage
\onecolumn

\section{Properties of ST K-means and Proofs}
\label{app:proofs}

\subsection{ST K-means Similarity Measure: a Quasipseudosemimetric}

\begin{prop}
The similarity measure defined by $\min_{\nu \in \mathbb{R}^{2\ell}} \| \mathcal{T}(x,\nu) - \mu \|$ is a Quasipseudosemimetric.
\end{prop}

\begin{proof}
Let's first define the orbit of an image with respect to the TPS transformations. Note that, the TPS does not form a group as it is a piecewise mapping. However, we know that it approximate any diffeomorphism on $\mathbb{R}^2$. Therefore, for sake of simplicity, we will make a slight notation abuse by considering the orbit, equivariance, and others group specific properties as being induced by the spatial transformer $\mathcal{T}$.
\begin{defn}
\label{def:orbit_clust}
We define the orbit an image $x$ under the action the $\mathcal{T}$ by
\begin{equation}
    \mathcal{O}(x) = \left \{\mathcal{T}(x,\nu) | \nu \in \mathbb{R}^{2 \ell} \right \}.
\end{equation}
\end{defn}

Let's now consider each metric statement:
1) It is non-negative as per the use of a norm.

2) \textbf{Pseudo:} $\min_{\nu \in \mathbb{R}^{2\ell}} \left \| (\mathcal{T}(x,\nu) - \mu \right \|=0 \Leftrightarrow \exists \nu \in \mathbb{R}^{2\ell}, s.t. \;\; x = \mathcal{T}(x,\nu) \Leftrightarrow  x \sim_{\mathcal{T}} \mu$, that is, $x$ and $\mu$ are equivariant with respect to the transformations induced by $\mathcal{T}$. Thus, $d(x,\mu)=0$ for possibly distinct values $x$ and $\mu$, however, these are not distinct when we consider the data as any possible point on their orbit with respect to the group of diffeomorphism. In fact, the distance is equal to $0$ if and only if, $\mu$ and $x$ are equivariant. 

3) \textbf{Quasi:} The asymmetry of the distance is due to the non-volume preserving deformations considered. In fact, we do not consider the Haar measure of the associated diffeomorphism group and consider the $L_2$ distance with respect to the Lebesgue measure. Although the asymmetry of $d$ does not affect our algorithm or results, a symmetric metric can be built by normalizing the distance by the determinant of the Jacobian of the transformation. Such a normalization would make the metric volume-preserving and as a result make the distance symmetric. 

4) \textbf{Semi:} If $x,x',x'' \in \mathcal{O}$, then $d(x,x'')= d(x,x')= d(x',x'') = 0$ as it exist a $\nu, \nu', \nu''$ such that the TPS maps each data onto the other as per definition of the orbit, thus the triangular inequality holds. If $x,x'' \in \mathcal{O}$ and $x' \notin \mathcal{O}$, we have $d(x,x'')= 0 \leq d(x,x') + d(x',x'') $. If  $x,x' \in \mathcal{O}$ and $x
''\notin \mathcal{O}$, we have $d(x',x'') = d(x,x'')$, and since $0 \leq d(x,x')$, the inequality is respected. However, if $x,x',x''$ belong to three different orbits, then we do not have the  guarantee that then triangular inequality holds. In fact, it will depend on the distance between the orbits which is specific to each dataset. 
\end{proof}

\subsection{ST K-means Updates: Proof of Proposition~\ref{prop2}}
We consider the F\'echet mean of the centroid $k$ to be the solution of the following optimization problem, $\argmin_{\mu_k} \sum_{i: x_i \in C_k} d(x_i,\mu_k)$. Using our similarity measure, we obtain the following.
\begin{proof} 
\label{proof:prop2}
The Fr\'echet mean for the cluster $C_k$ is defined as $\argmin_{\mu_k} \sum_{i: x_i \in C_k} \left \|  \mathcal{T}(x_i,\nu^{\star}) - \mu_k  \right \|^{2}$
since the optimization problem is convex in $\mu_k$ (as the result of the composition of the identity map and a norm which are both convex) we have $\mu_k^{\star}: \nabla_{\mu} \sum_{i: x_i \in C_k} \left \| \mathcal{T}(x_i,\nu^{\star}) - \mu_k  \right \|^{2}= 0$.
with, 
\begin{align}
    \nabla_{\mu}  \sum_{i: x_i \in C_k} \left \|  \mathcal{T}(x_i,\nu^{\star}) -  \mu_k  \right \|^{2}  & = 2 \: (\left | C_k \right | ) \times  \mu_k  + 2 \sum_{i: x_i \in C_k} \mathcal{T}(x_i,\nu^{\star}).
\end{align}
\end{proof}

\subsection{ST K-means Similarity Measure: Invariance Property}
\label{app:invariance}
Motivated by the fact that small non-rigid transformations, usually, do not change nature of an image, we propose to exploit the invariance property of the similarity measure we proposed. 

In this section, for sake of simplicity we will assume that the transformations belong to the group of diffeomorphism. In practice, the TPS can only approximate element of such group, and the constraint we impose on the transformation, e.g., number of landmark, also limit the type of diffeomorphism that can be approximated, therefore, we could instead consider that we approximate a subgroup of the diffeomorphism group.

Let's define an invariant similarity measure under the action of such group. That is, the similarity between two $2$-dimensional signals remain the same under any diffeomorphic transformations. We propose to define the invariance in the framework of centroid-based clustering algorithm as follows.
\begin{defn}
\label{def:inv}
An invariant similarity measure with respect to $\text{diff}(\mathbb{R}^2)$ is defined as $d: \mathbb{R}^n \times \mathbb{R}^n \rightarrow \mathbb{R}^{+}$ such that for all images $x \in \mathbb{R}^n$, all centroids $\mu \in \mathbb{R}^2$, and all group elements $\forall g \in \text{diff}(\mathbb{R}^2)$, we have
\begin{equation}
\label{eq:inv}
d(x,\mu) = d( g \star x, \mu),
\end{equation}
where $g \star x$ denotes the action of the group element $g$ onto the image $x$.
\end{defn}

The similarity used in Eq.~\ref{eq:def_invariant} of the optimization problem is $\text{diff}(\mathbb{R}^2)$-invariant as per Definition~\ref{def:inv}.
\begin{prop}
The similarity $ \min_{g \in \text{diff}(\mathbb{R}^2)} \left \| g \star x - \mu \right \|$ is $\text{diff}(\mathbb{R}^2)$-invariant.
\label{prop1}
\end{prop}
\begin{proof} 
\label{proof:prop}

Let consider $g^{\star} = \argmin_{g \in \text{diff}(\mathbb{R}^2)} \left \| g \star x - \mu \right \|$, we have $\argmin_{g \in \text{diff}(\mathbb{R}^2)} \left \| g \cdot g' \star x - \mu \right \| = g^{\star} \cdot g'^{-1}$, where $g'^{-1}$ is the inverse group element of $g'$. In fact, $\left \| g^{\star} \cdot g'^{-1} \cdot g' \star x - \mu \right \| = \left \| g^{\star} \star x - \mu \right \|.$ Since for all $g' \in \text{diff}(\mathbb{R}^2)$, it exists an inverse element $g'^{-1}$, we have that $\forall g \in \text{diff}(\mathbb{R}^2)$, $d( g' \star x, \mu) = d(x,\mu)$.

That is, by definition of the group, there is always another element that minimizes the loss function by using the composition between the inverse element of the group that has just been added, $g'$, and the optimal element $g^{\star}$.
\end{proof}

\subsection{ST $K$-means Convergence: Proof of Proposition~\ref{prop:convergence}}

\begin{proof}
Following the notation of Sec.~\ref{app:TPS}, we can define the spatial transformer operator $\mathcal{T}$ as the composition of the TPS and bilinear interpolation map. That is, $\mathcal{T}(x,\nu)=\Gamma[F(\nu),x]$. Now the aim is to prove that $\min_{\nu \in \mathcal{R}^{2l}} \left \| \mathcal{T}(x,\nu) - \mu \right \|_2^2$ defines a Bregman divergence measure as in \cite{banerjee2005clustering}. In such a case, Algo.~\ref{algo:algo} defines a special case of the Bregman divergence hard-clustering algorithm again defined in \cite{banerjee2005clustering} which is proven to converge. 

Let's first start by making an assumption on the data $x$, we can without loss of generality assume that they are non-negative as we are dealing either with images or time-frequency representation where a modulus is applied to obtain the $2$-dimensional real representation. Then, we also assume that the minimum over the transformation parameters $\nu$ reaches a global unique minimum, denoted by $\nu^{\star}$.
Now,
\begin{align*}
    \left \| \mathcal{T}(x,\nu^{\star}) - \mu \right \|_2^2 &= \langle  \mathcal{T}(x,\nu^{\star}), \mathcal{T}(x,\nu) \rangle + \langle \mu, \mu \rangle - 2 \langle \mathcal{T}(x,\nu^{\star}), \mu \rangle \nonumber \\
    & = \langle \mathcal{T}(x,\nu^{\star}),\mathcal{T}(x,\nu^{\star}) \rangle - \langle \mu, \mu \rangle - \langle \mathcal{T}(x,\nu^{\star}) - \mu, 2 \mu \rangle~~,
\end{align*}
Now it is clear that $\mu =  \mathcal{T}(\mu,0)$ which consists in the identity transform of the centroid $\mu$. Then we denote by $\phi_{\nu^{\star}}(y) = \langle \mathcal{T}(y,\nu^{\star}), \mathcal{T}(y,\nu^{\star}) \rangle$ where $y \in \mathbb{R}^n$, and obtain that,
\begin{equation*}
    \left \| \mathcal{T}(x,\nu^{\star}) - \mu \right \|_2^2 = \phi_{\nu^{\star}}(x) - \phi_{0}(\mu) - \langle \mathcal{T}(x,\nu^{\star}) - \mathcal{T}(\mu,0), \nabla \phi_{0}(\mu) \rangle~~. 
\end{equation*}
Now, we know that $\langle x,x \rangle $ is non-decreasing w.r.t each dimension since the image or time-frequency representation are positive real valued, and the inner product defines a strictly convex map. Then, we also know that $\mathcal{T}(x,\nu^{\star})=\Gamma[F(\nu^{\star}),x]$ is defined as the composition of the TPS for the coordinate and the bilinear map for the image, which can be formulated as a linear transformation with respect to the data $x: Ax$, where $A$ is a structured sparse matrix where each block denotes the dependency to nearby pixels. Therefore this mapping is convex. As a composition between a non-decreasing w.r.t each dimension and strictly convex function with a convex function, $\phi_{\nu^{\star}}$ is strictly convex, which complete the proof. 

\end{proof}

\subsection{ST $\boldsymbol{K}$-means: Weighted Voronoi Diagram}
\label{app:metric}

\begin{proof}
\label{proof:voronoi}
Let's start by re-writting the similarity measure as to analytically express a metric tensor that would be the weight in the weighted Voronoi diagram the ST $K$-means defines. Using App.~\ref{app:TPS}, we can re-write $d(x,\mu_k) = \min_{\nu \in \mathbb{R}^{2 \ell}} \left \| \mathcal{T}(x,\nu) - \mu_k \right \|_2^2 = \min_{\nu \in \mathbb{R}^{2\ell}} \left \| A(\nu) x - \mu_k \right \|_{2}^2$, where $A(\nu)$ is bilinear in the coordinates that are induced by the TPS. In this formulation we can observe that $\nu$ defines the displacement vector w.r.t the original uniform grid of landmark. That is, if $\nu$ is the null vector, then $A(\nu)x=x$. Now, we assume that such linear operator is inversible, i.e., the TPS transformation is invertible (note that this is not always the case \cite{johnson2001landmark}). Then, we can re-write $d(x,\mu_k)$ as
\begin{equation*}
    \min_{\nu \in \mathbb{R}^{2 \ell}} \left \| x - A(\nu_{x,k})^{-1} \mu_k \right \|_{A(\nu_{x,k})^TA(\nu_{x,k})}^2,
\end{equation*}
where $\left \| x \right \|_{A(\nu_{x,k})^T A(\nu_{x,k})} = x^T A(\nu_{x,k})^T A(\nu_{x,k}) x$, and $A(\nu_{x,k})^T A(\nu_{x,k})$ defines the metric tensor, and the notation $\nu_{x,k}$ indicates that the displacement vector $\nu$ depends on the centroid $\mu_k$ and the datum $x$. Also, note that while $A(\nu_{x,k})$ defines the transformation operator to map $x$ onto $\mu_k$, $A(\nu_{x,k})^{-1}$ is the inverse operator mapping the centroid $\mu_k$ to the datum $x$.

Now the tuple of cells $\left \{R_k \right \}_{k=1}^{K}$ defines such as
\begin{equation*}
    R_k = \left \{ x \in \mathbb{R}^n | \left \|x - A(\nu_{x,k})^{-1} \mu_k \right \|_{A(\nu_{x,k})^TA(\nu_{x,k})} \leq  \left \|x - A(\nu_{x,j})^{-1} \mu_j \right \|_{A(\nu_{x,j})^TA(\nu_{x,j})},~~ \forall j \neq k \right \}~~, 
\end{equation*}
defines a weighted Voronoi diagram \cite{letscher2007vector,inaba1994applications}, where we observe that the metric tensor is dependant on all the spatial transformations. 
\end{proof}

\section{Implementation Details}
\label{app:impl_details}
Note that the deformation invariant similarity measure we introduced in Eq.~\ref{eq:def_invariant} differs from the affine invariant distances developed in \cite{fitzgibbon2002affine,simard2012transformation,lim2004image}. All previously defined measures of error rely on the assumption that the manifold can be locally linearized and as a result the tangent space is used as a proxy to learn the optimal affine transformation. 
However, the work of \cite{wakin2005multiscale} suggests that tangent planes fitted to image manifold continually twist off into new dimensions as the parameters of the affine transformations vary due to a possible the intrinsic multiscale structure of the manifold. As such, the alignment of two images can be done by linearizing the manifold. To do so, \cite{wakin2005multiscale} propose to consider the multiscale structure of the manifold, we simplify their approach by applying a low-pass filter on the images and the centroid prior to learn the affine transformation best aligning them. Then, we optimize the remaining part of the TPS to account for diffeomorphic transformations. These two steps are similar to the one used in \cite{jaderberg2015spatial}.

\section{Alternative Methods}
\label{app:compete}
We compare our model with well-known clustering techniques using deep neural networks. We performed experiments for the VaDE \cite{jiang2016variational} and DEC \cite{xie2016unsupervised} using the code made publicly available by the authors. 
We use the annotation (MLP) as a reference to the MLP architecture used in their experiments (see Appendix~\ref{app:arch} for details). 
To fairly compare our model to the DEC and VaDE models, we proposed a convolutional architecture to the DEC and VaDE networks, denoted by DEC (Conv) and VaDE (Conv) (see Appendix~\ref{app:arch} for details). 
Finally, we evaluate the performance of an augmented $K$-means algorithm trained using the features extracted by an Autoencoder, denoted by AE $+$ $K$-means in the following. 

The parameters of the different models mentioned above are learned by stochastic gradient descent (Adam optimizer \cite{kingma2014adam}). In all the experiments, the learning rate are cross-validated following the approach in \cite{xie2016unsupervised} according to $[10^{-4}, 5\times10^{-4}, 10^{-3}, 5\times 10^{-3}, 10^{-2},5\times 10^{-2} ]$. The internal parameters that are model dependent, e.g., the number of pre-training epoch and the update intervals, are also cross-validated.  

We also compare our ST $K$-means to the closely related $K$-means and affine invariant $K$-means, denoted by AI $K$-means.%
For each run, all three $K$-means algorithms start from the same initial centroids using the $K$-means$++$ algorithm developed by \citet{arthur2006k} to speed up the convergence of the $K$-means algorithm.

\section{Neural Network Architectures}
\label{app:arch}

For both architectures , the decoder architecture is symmetric to the encoder and the batch size is set to $64$. 
\paragraph{MLP:} The MLP architecture from input data to bottleneck hidden layer is composed of $4$ fully connected ReLU layers with dimensions $\left [500,500,2000,10 \right ]$.

\paragraph{Conv:} The CONV architecture is composed of $3$ $2d$-convolutional ReLU layers with $32$ filters of size $5\times5$, and $2$ fully connected ReLU layers with dimension $\left [400, 10 \right ]$. For each layer, a batch normalization is applied.

\section{Thin-Plate-Spline Interpolation}
\label{app:TPS}
Let's consider two set of landmarks, the source ones $\nu_s = \{u_i,v_i\}_{i=1}^\ell$ and the transformed $\nu_t = \{u_i',v_i'\}_{i=1}^\ell$ where $\ell$ denotes the number of landmarks. The TPS aim at finding a mapping $F=(F_1,F_2)$, such that $F(u,v) = (F_1(u,v),F_2(u,v)) = (u',v')$, that is, the mapping between two set of landmarks. The particularity of the TPS is that it learns such a mapping by minimizing the interpolation term, and a regularization that consists in penalizing the bending energy. 

The TPS optimization problem is defined by
\begin{equation}
    \min_{F} \sum_{i=1}^{N} \left \|(u_i',v_i')-F(u_i,v_i)  \right \|^2 + \lambda \int \int \left [ (\frac{\partial ^2 F}{\partial u^2})^2 + 2 (\frac{\partial^2 F}{\partial u \partial v})^2 + (\frac{\partial ^2 F}{\partial v^2})^2  \right ] du dv.
\label{eq:bend}
\end{equation}

In our model, the source landmarks are considered to be the coordinates of a uniform grid. Also note that both the source landmarks and transformed ones are usually a subset of the set of coordinate of the images. For instance, for the MNIST dataset of size $28\times28$, the landmarks would be a grid of size $\ell \times \ell$, where $\ell <28$. While the mapping is based on the landmark, it is then applied to the entire image coordinate. In fact, $F=(F_1,F_2)$ is mapping $\mathbb{R}^2 \rightarrow \mathbb{R}^2$, where  $F_1$ (resp. $F_2$) corresponds to the mapping from $(x,y)$ to the first dimension $x'$ (resp. the second dimension $y'$).

The solution of the TPS optimization problem, Eq.~\ref{eq:bend}, provides the following analytical formula for $F$
\begin{align}
\label{eq:tpsx}
     F_1(u,v) = \hspace{-.1cm} u'   \hspace{-.1cm} =   \hspace{-.1cm} a_1^{(1)} &  \hspace{-.1cm} + \hspace{-.1cm} a_{u}^{(1)}u \hspace{-.1cm}  + \hspace{-.1cm}  a_{v}^{(1)} v \hspace{-.1cm}  + \hspace{-.1cm} \sum_{i=1}^{\ell} \hspace{-.1cm}  w_i^{(u)}   U(\left | (u_i,v_i)\hspace{-.1cm}  -\hspace{-.1cm}  (u,v) \right |),
\end{align}
\begin{align}
\label{eq:tpsy}
     F_2(u,v) = \hspace{-.1cm} v'  \hspace{-.1cm} =   \hspace{-.1cm} a_1^{(2)} & \hspace{-.1cm}  + \hspace{-.1cm} a_{u}^{(2)}u \hspace{-.1cm} + \hspace{-.1cm} a_{v}^{(2)} v \hspace{-.1cm}  +\hspace{-.1cm}  \sum_{i=1}^{\ell} \hspace{-.1cm}  w_i^{(v)}  U(\left | (u_i,v_i) \hspace{-.1cm} - \hspace{-.1cm} (u,v) \right |),
\end{align}
where $\left | . \right |$ is the $L_1$-norm, $a_1,a_{u},a_{v}$ are the parameters governing the affine transformation, and $w_i$ are parameters responsible for non-rigid transformations as they stand as a weight of the non-linear kernel $U$. The non-linear kernel $U$ is expressed by $U(r) = r^2 \log(r^2), \forall r \in \mathbb{R}_{+}$.

Based on the landmarks $\nu_s$ and $\nu_t$, we can obtain these parameters by solving a simple system of equation define by the following operations

\begin{equation}
\label{eq:solving}
    \mathcal{L}^{-1}\mathcal{V} = \begin{bmatrix}
(W^{(x)}|a_1^{(x)} a_x^{(x)} a_y^{(x)})^T\\ 
(W^{(x)}|a_1^{(y)} a_x^{(y)} a_y^{(y)})^T
\end{bmatrix}.
\end{equation}
where the matrix $\mathcal{L} \in \mathbb{R}^{(\ell+3) \times (\ell+3)}$, is defined as
\[ \mathcal{L} =
\left[
\begin{array}{c|c}
\mathcal{K} & \mathcal{P} \\
\hline
\mathcal{P}^T & \mathcal{O}
\end{array}
\right],
\mathcal{K} = \begin{bmatrix}
 0  & U(r_{12}) & \dots & U(r_{1\ell}) \\ 
 U(r_{21}) & 0 & \dots & U(r_{2\ell}) \\ 
 \dots & \dots & \dots & \dots \\ 
 U(r_{\ell1}) & \dots & \dots & 0 
\end{bmatrix}, 
    \mathcal{P} = \begin{bmatrix}
 1  & x_1  & y_1 \\ 
 1 &  x_2 & y_2 \\ 
 \dots & \dots & \dots  \\ 
 1 & x_{\ell} & y_{\ell} 
\end{bmatrix} \]
where $r_{ij}= \left | (u_i,v_i) - (u_j,v_j) \right |$, $\mathcal{K} \in \mathbb{R}_{+}^{\ell \times \ell}$,
and $\mathcal{V}=\begin{bmatrix}
x_1' & x_2' & \dots & x_\ell'| 0 &0 &0\\ 
y_1' & y_2' & \dots & y_\ell'| 0 &0 &0
\end{bmatrix}.$

Note that, since the matrix $\mathcal{L}$ depends only on the source landmarks, and that in our case these are unchanged, its inverse can be computed only once. The only operation required to be computed for each data and each centroid is the matrix multiplication $\mathcal{L}^{-1} \mathcal{V}$ providing the parameters of the TPS transformation, as per Eq.~\ref{eq:tpsx}, \ref{eq:tpsy}. Given these parameters, the mapping $F$ can be applied to to each coordinate of the image.

 Now in order to render the image, one can perform bilinear interpolation as it is achieved in. Besides, the bilinear interpolation will allow the propagation of the gradient through any differentiable loss function.

Given an image $x_1 \in \mathbb{R}^{n}$  where $n=W \times H$, $W$ denotes the width and $H$ the height of the image, and two sets of landmarks $\nu_s = \{u_i,v_i\}_{i=1}^\ell$ ,uniform grid coordinate of $x_1$, and $\nu_t = \{u_i',v_i'\}_{i=1}^\ell$, the transformation of the uniform grid, which are a subset of the image coordinate,
We are able to learn a mapping $F=(F_1,F_2)$ such that for each original pixel coordinate, we have their transformed coordinates. In fact, given any position $(u,v)$ on the original image, the mapping $F$ provides the new positions $(u',v')$ as per Eq.~\ref{eq:tpsx}, Eq.~\ref{eq:tpsy}.

Now, from this transformed the coordinates space, we can render an image $x_2 \in \mathbb{R}^{n}$ using, as in \cite{jaderberg2015spatial}, the bilinear interpolation function $\Gamma: \mathbb{R}^2 \times \mathbb{R}^n \rightarrow \mathbb{R}$ which takes as input the original image $x_1$ and the transformed pixel coordinates $(u',v')$, and outputs the pixel value of the transformed image at a given pixel coordinate 
 \begin{equation*}
 \begin{aligned}
     x_2  (k,l)= & \Gamma[F(u_k,v_l),x_1] \nonumber \\
      = & \Gamma[(u_k',v_l'),x_1] \nonumber \\
      = & \sum_{t,h \in \{0,1\}}\sum_{i=1}^{W} \sum_{j=1}^{H} x_1(i,j) \delta(\floor{u_k'+t}-i) 
       \times \delta(\floor{v_l'+h}-j)(u_k'-\floor{u_k'})^{\delta(t)}(v_l'-\floor{v_l'})^{\delta(h)} \nonumber \\
       & \hspace{1cm} \times (1-(v_l'-\floor{v_l'}))^{\delta(t-1)}(1-(u_k'-\floor{u_k'}))^{\delta(h-1)}, 
     \end{aligned}
 \end{equation*}
 where $\delta$ is the Kronecker delta function and $\floor{.}$ is the floor function rounding the real coordinate to the closest pixel coordinate. 
 

\section{Datasets}
\label{ap:data}
\noindent
\textbf{MNIST \cite{deng2012mnist}:} is a handwritten digit dataset containing $60.000$ training and  $10.000$ test images of dimension $28 \times 28$ representing $10$ classes. 
\\
\noindent
\textbf{Rigid MNIST:} we randomly sample one instance of each MNIST class and generate $100$ random affine transformations for each sample.  A third of the data are used for testing.
\\
\noindent
\textbf{Non-rigid MNIST:} we randomly sample one instance of each MNIST class and generate $100$ random affine transformations for each sample as well as random transformations using the TPS method. A third of the data are used for testing.
\\
\noindent
\textbf{Audio MNIST \cite{becker2018interpreting}:} is composed of $30000$ recordings of spoken digits by  $60$  different speaker and sampled at $48kHz$ of $1$sec long. It consists of $10$ classes. We use $10000$ data for testing and $20000$ for testing. This dataset will be transformed into a time-frequency representation. This representations, can be considered as images, are usually used as the common representation of audio recordings \cite{cosentino2020}. 
\\
\noindent
\textbf{E-MNIST \cite{cohen2017emnist}:} is a handwritten letters dataset merging a balanced set of the uppercase and lowercase letters into a single $26$ classes dataset of dimension $28 \times 28$.
\\
\noindent
\textbf{Rock-Paper-Scissors \cite{rps}:} images of hands playing rock, paper, scissor game, that is, a $3$ classes dataset. The dimension of each image is $300 \times 300$. The training set is composed of $2520$ data and the testing set $372$.
\\
\noindent
\textbf{Face-10 \cite{gourier2004estimating}:} images of the face of 15 people, wearing glasses or not and having various skin color. For each individual, different samples are obtained with different pose orientation varying from $-90$ degrees to $+90$ degrees vertical degrees. The dimension of each image is $288 \times 384$. The training set is composed of $273$ data and testing set of $117$ data.
\\
\noindent
\textbf{Arabic Char \cite{altwaijry2020arabic}:} Handwritten Arabic characters written by $60$ participants. The dataset is composed of $13,440$ images in the training set and $3360$ in the test set. The dimension of each image is $32 \times 32$.

\section{Supplementary Visualization}

\subsection{Additional t-SNE Visualisations}
\label{app:sup_tsne}

\begin{figure*}[h!]

\begin{minipage}{.32\linewidth}
    \centering
\textbf{\textit{Raw Data}}
\end{minipage}
\begin{minipage}{.32\linewidth}
    \centering
\textbf{\textit{Affine Invariant}}
\end{minipage}
\begin{minipage}{.32\linewidth}
    \centering
\textbf{\textit{Spatial Transformer}}
\end{minipage}

\begin{minipage}{.01\linewidth}
\rotatebox{90}{\textbf{\textit{E-MNIST}}}
\end{minipage}
\begin{minipage}{.32\linewidth}
    \centering
    \includegraphics[width=\linewidth]{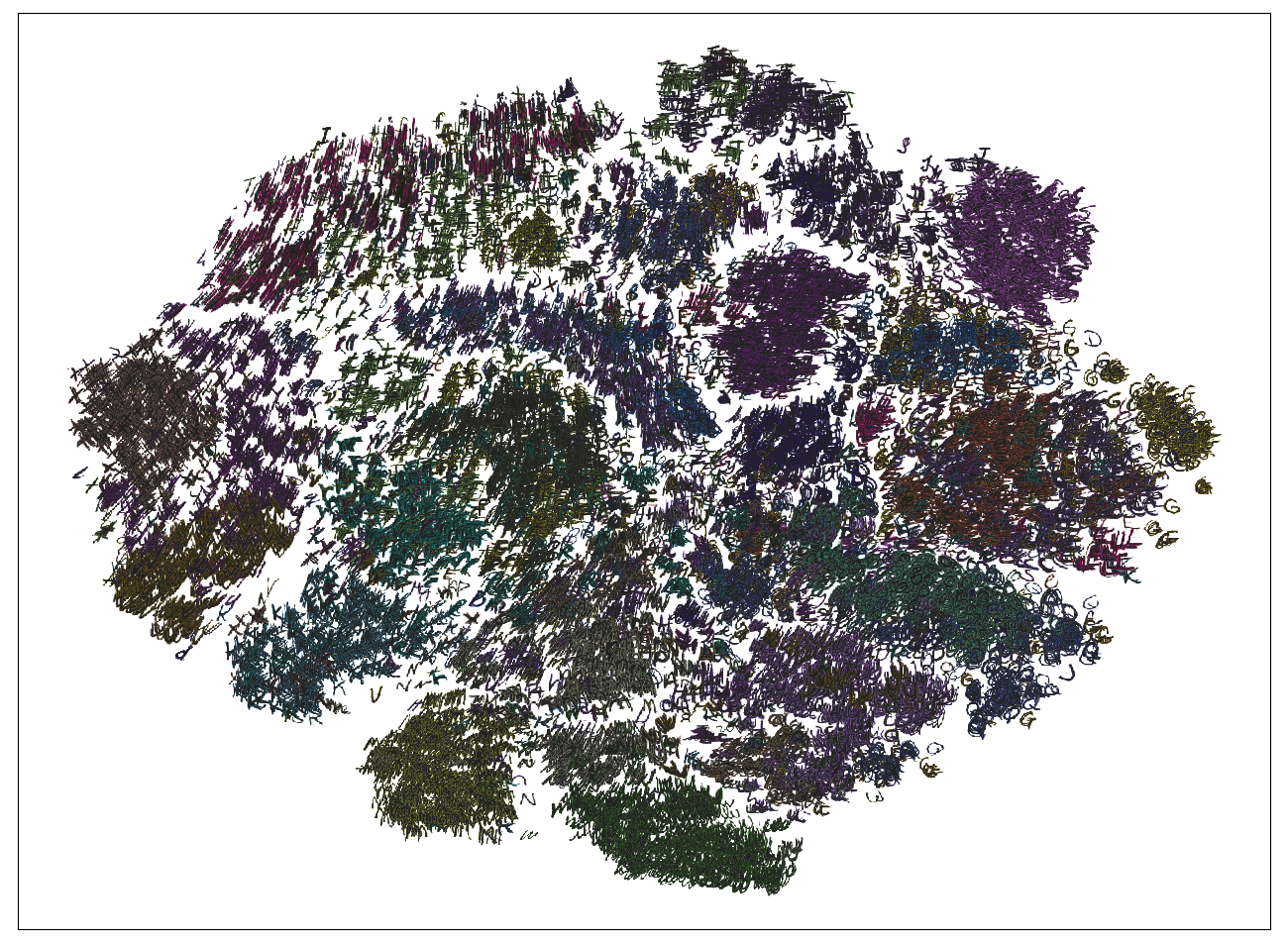}
\end{minipage}
\begin{minipage}{.32\linewidth}
    \centering
    \includegraphics[width=\linewidth]{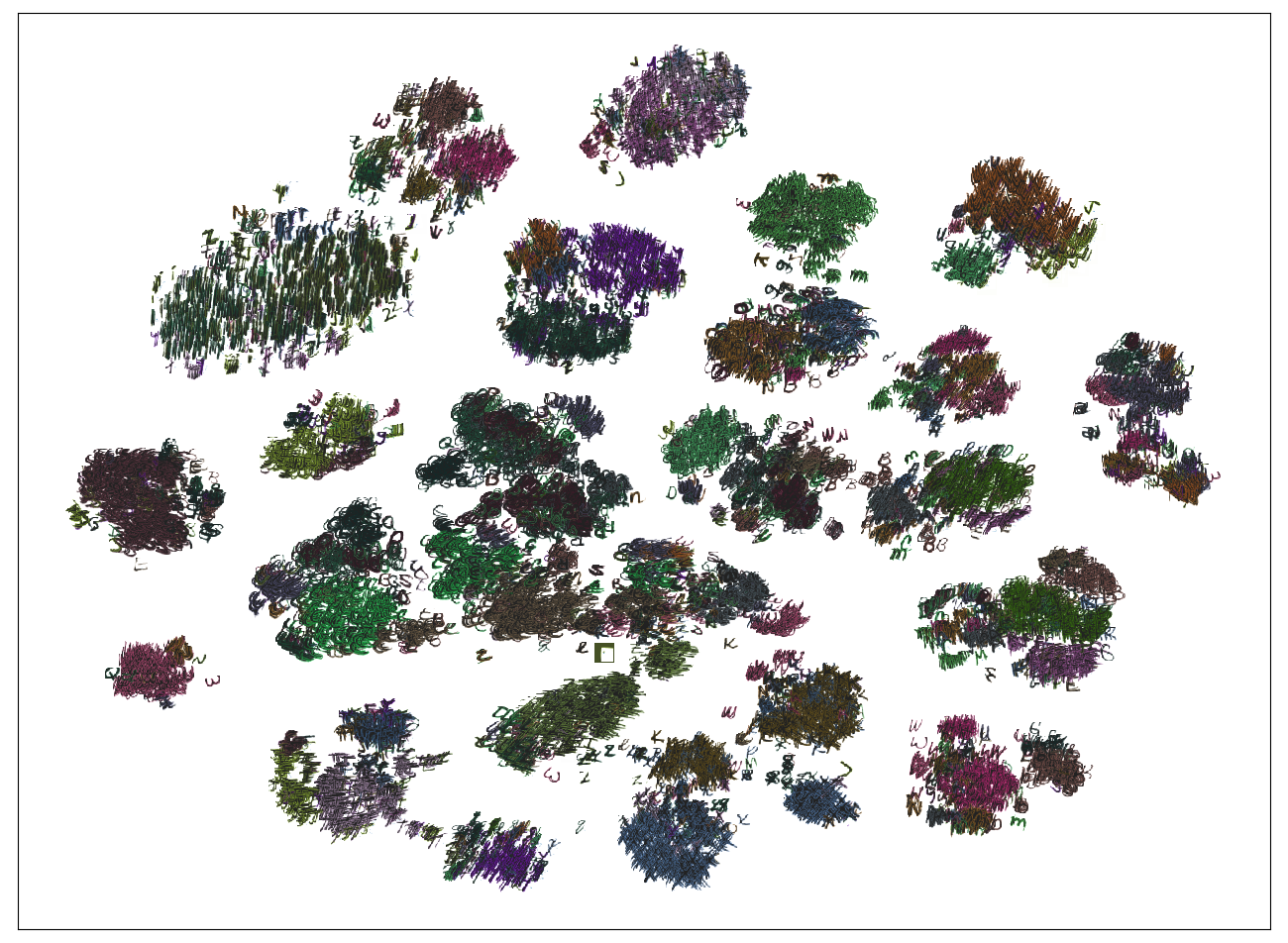}
\end{minipage}
\begin{minipage}{.32\linewidth}
    \centering
    \includegraphics[width=\linewidth]{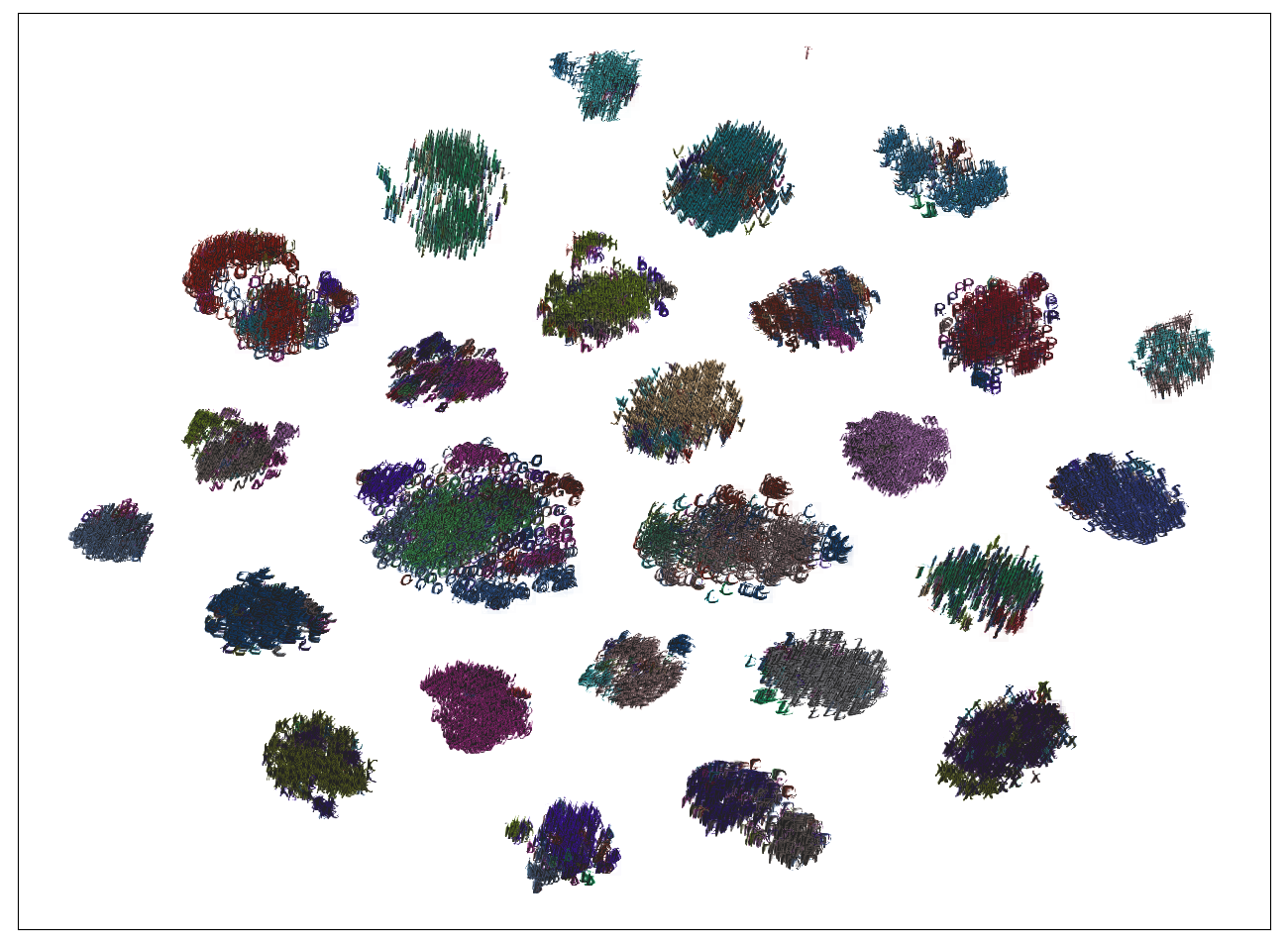}
\end{minipage}

\begin{minipage}{.01\linewidth}
\rotatebox{90}{\textbf{\textit{Arabic Characters}}}
\end{minipage}
\begin{minipage}{.32\linewidth}
    \centering
    \includegraphics[width=\linewidth]{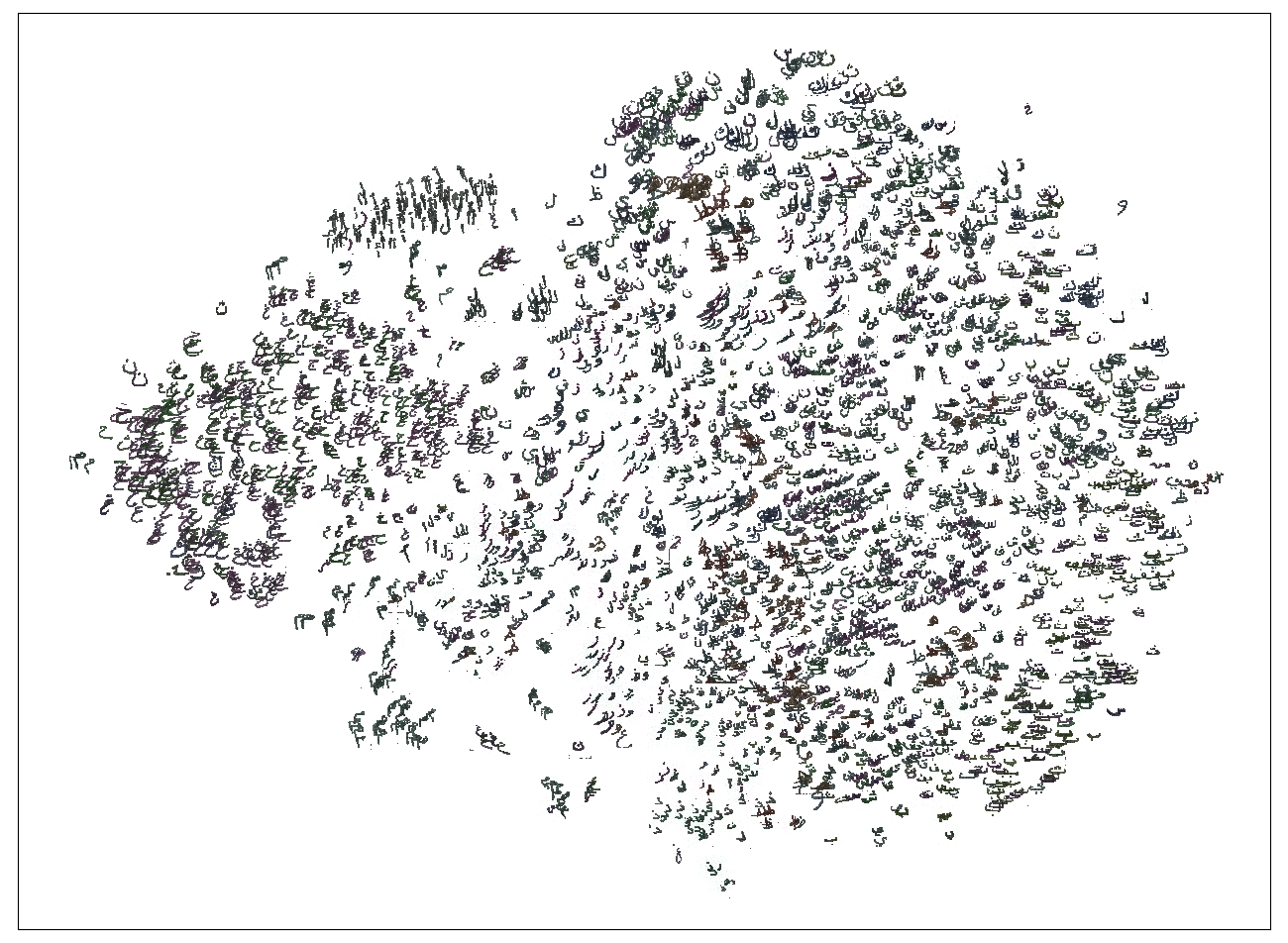}
\end{minipage}
\begin{minipage}{.32\linewidth}
    \centering
    \includegraphics[width=\linewidth]{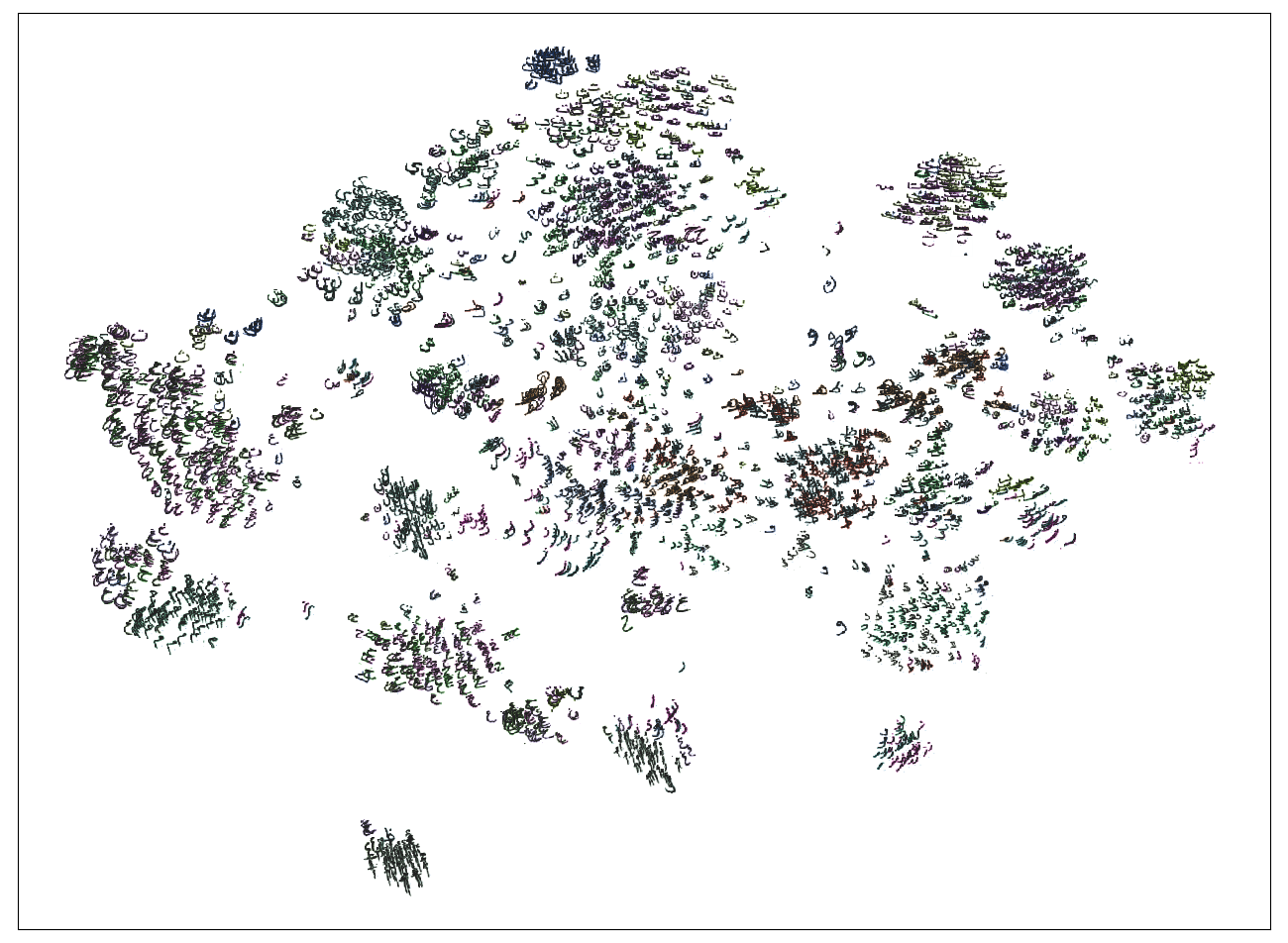}
\end{minipage}
\begin{minipage}{.32\linewidth}
    \centering
    \includegraphics[width=\linewidth]{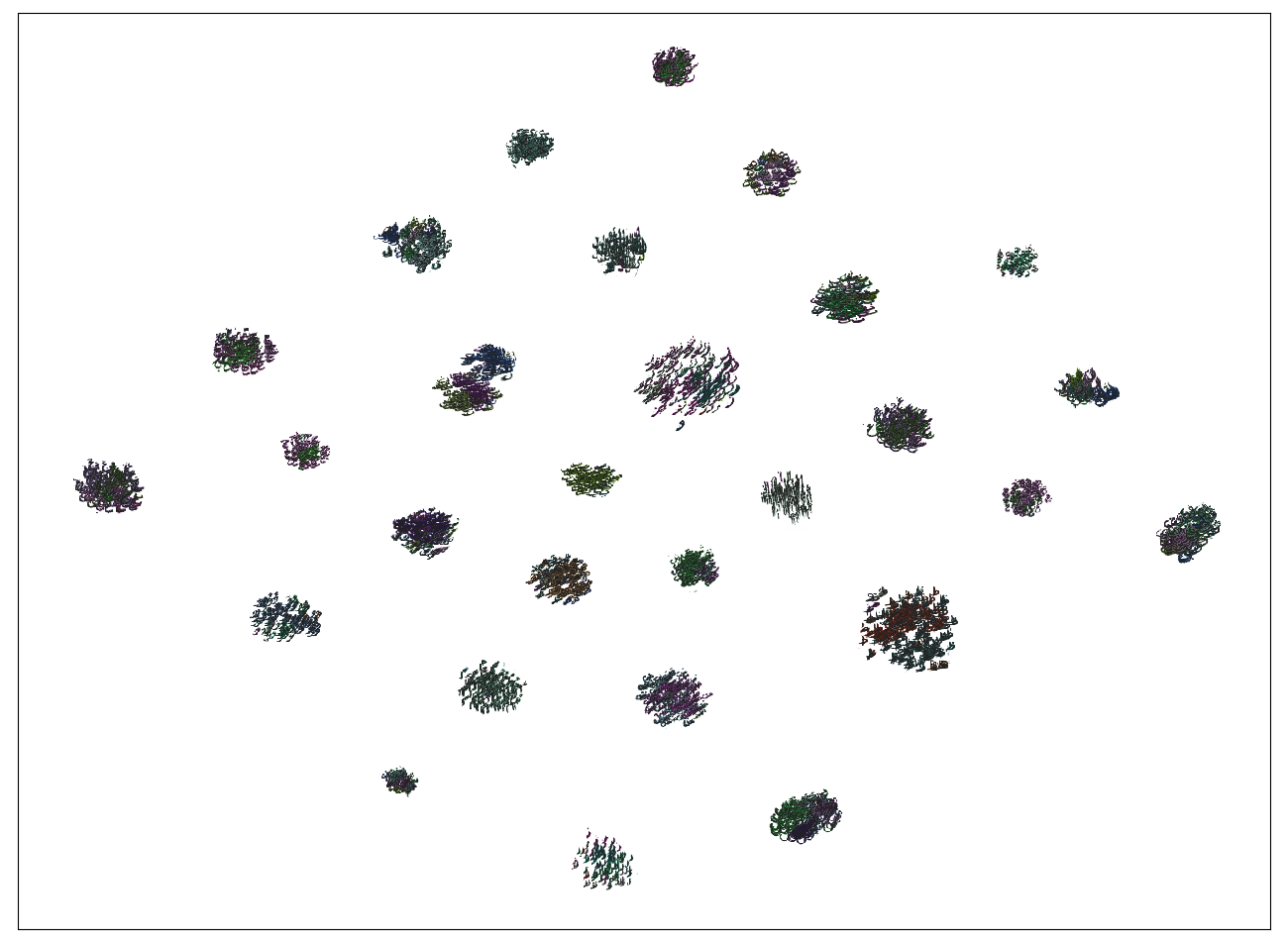}
\end{minipage}

\caption{\textbf{$2$-dimensional t-SNE Vizualisation} - The raw data (\textit{left column}), the affinely transformed data, i.e., we extract the transformation of the data that corresponds to the centroid it was assigned and perform the t-SNE on these affinely transformed data, (\textit{middle column}), the data transformed with respect to  non-rigid transformations as per Eq.~\ref{eq:def_invariant}, i.e., the same process as previously mention, but we consider the transformation induced by the TPS, and then perform the dimension reduction on these transformed data, (\textit{right column}). Each row corresponds to a different dataset, E-MNIST, Arabic Characters,  are depicted from the top to bottom row. For all the figures, the colors of the data represent their ground truth labels.}
\end{figure*}

\subsection{Additional Centroid Visualisations}
\label{app:sup_centroids}

\begin{figure*}[h!]
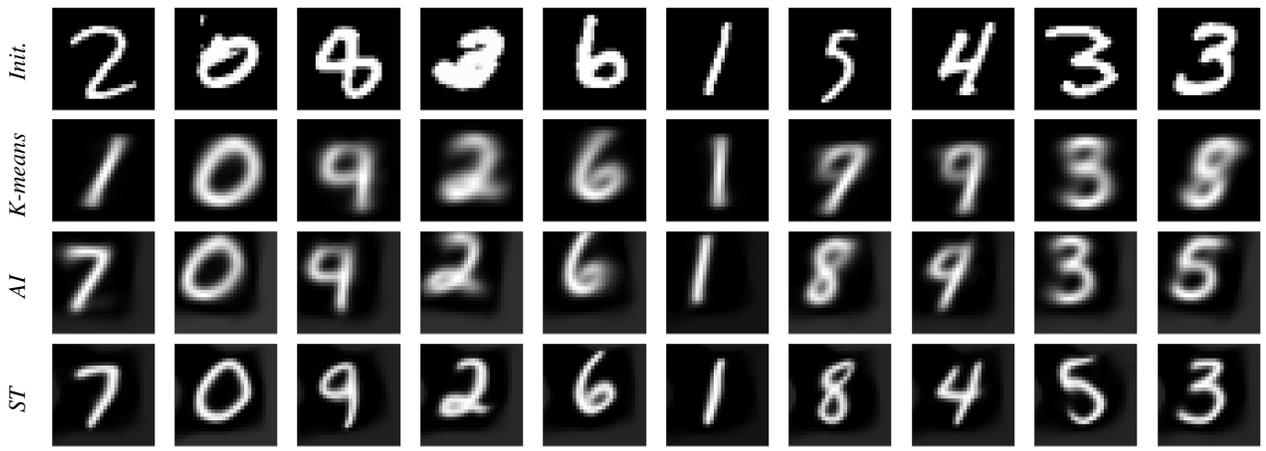

\begin{minipage}{.02\linewidth}
\rotatebox{90}{\small \textit{ST  \; \; \; \; \;\;\;\;\;AI \;\;\;\;\;\;\; K-means\;\;\;\;\;\;\;\;Init.}}
\end{minipage}
\foreach \c in {0,...,9}{
    \begin{minipage}{0.085\linewidth}
    \includegraphics[width=\linewidth]{images/best_mnist_centroid_init\c.png}\\
    \includegraphics[width=\linewidth]{images/kmean_mnist_centroid_final\c.png}\\
    \includegraphics[width=\linewidth]{images/aff_mnist_centroid_final\c.png}\\
    \includegraphics[width=\linewidth]{images/best_mnist_centroid_final\c.png}
    \end{minipage}
}
\caption{ \textit{Additional} - \textbf{MNIST Centroids Visualization (dim 28x28)} - Depiction of the initialization of the per-cluster centroids \textbf{top row} and the final per-cluster centroids of the K-means, Affine invariant K-means, and ST K-means (proposed) methods.}
\end{figure*}

\begin{figure*}[h!]
\begin{minipage}{.02\linewidth}
\rotatebox{90}{\small \textit{ST  \; \; \; \; \;\;\;\;\;AI \;\;\;\;\;\;\; K-means\;\;\;\;\;\;\;\;Init.}}
\end{minipage}
\foreach \c in {10,...,19}{
    \begin{minipage}{0.085\linewidth}
    \includegraphics[width=\linewidth]{images/best_emnist_centroid_init\c.png}\\
    \includegraphics[width=\linewidth]{images/kmean_emnist_centroid_final\c.png}\\
    \includegraphics[width=\linewidth]{images/aff_emnist_centroid_final\c.png}\\
    \includegraphics[width=\linewidth]{images/best_emnist_centroid_final\c.png}
    \end{minipage}
}
\caption{ \textit{Additional} - \textbf{$10$ out of $26$ E-MNIST Centroids Visualization (dim 28x28)} - Depiction of the initialization of the per-cluster centroids \textbf{top row} and the final per-cluster centroids of the K-means, AI K-means, and ST K-means (proposed) methods.}
\end{figure*}

\begin{figure*}[h!]
\begin{minipage}{.02\linewidth}
\rotatebox{90}{\small \textit{ST  \; \; \; \; \;\;\;\;\;AI \;\;\;\;\;\;\; K-means\;\;\;\;\;\;\;\;Init.}}
\end{minipage}
\foreach \c in {0,...,9}{
    \begin{minipage}{0.085\linewidth}
    \includegraphics[width=\linewidth]{images/best_audiomnist_centroid_init\c.png}\\
    \includegraphics[width=\linewidth]{images/kmean_audiomnist_centroid_final\c.png}\\
    \includegraphics[width=\linewidth]{images/aff_audiomnist_centroid_final\c.png}\\
    \includegraphics[width=\linewidth]{images/best_audiomnist_centroid_final\c.png}
    \end{minipage}
}
\caption{ \textit{Additional} - \textbf{Audio MNIST Centroids Visualization (dim 64x24)} - Depiction of the initialization of the per-cluster centroids \textbf{top row} and the final per-cluster centroids of the K-means, AI K-means, and ST K-means (proposed) methods.}
\end{figure*}

\begin{figure*}[h!]
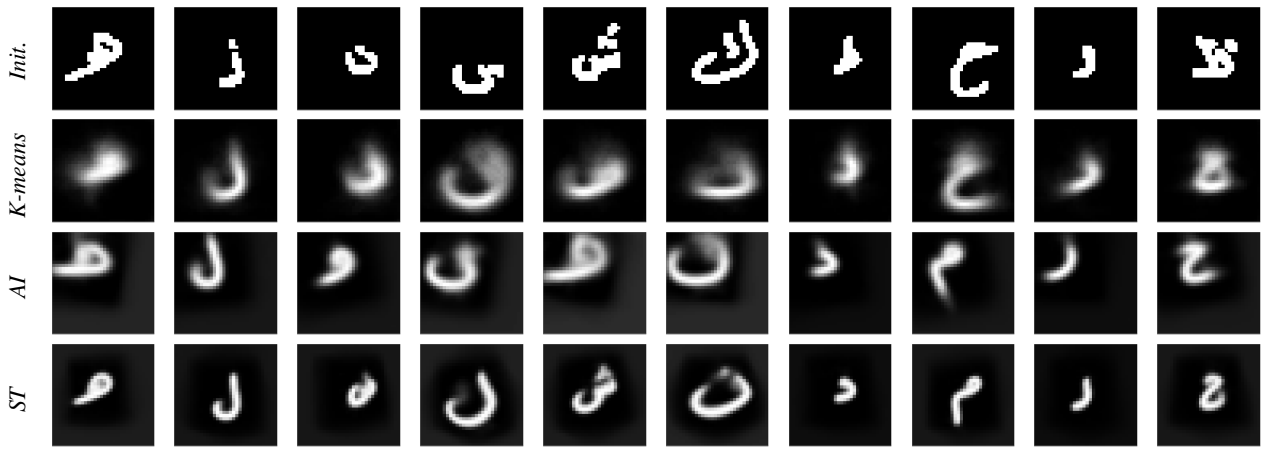

\begin{minipage}{.02\linewidth}
\rotatebox{90}{\small \textit{ST  \; \; \; \; \;\;\;\;\;AI \;\;\;\;\;\;\; K-means\;\;\;\;\;\;\;\;Init.}}
\end{minipage}
\foreach \c in {0,...,9}{
    \begin{minipage}{0.085\linewidth}
    \includegraphics[width=\linewidth]{images/best_arabchar_centroid_init\c.png}\\
    \includegraphics[width=\linewidth]{images/kmean_arabchar_centroid_final\c.png}\\
    \includegraphics[width=\linewidth]{images/aff_arabchar_centroid_final\c.png}\\
    \includegraphics[width=\linewidth]{images/best_arabchar_centroid_final\c.png}
    \end{minipage}
}
\caption{ \textit{Additional} - \textbf{$10$ out of $28$ Arab Characters Centroids Visualization (dim 32x32)} - Depiction of the initialization of the per-cluster centroids \textbf{top row} and the final per-cluster centroids of the K-means, AI K-means, and ST K-means (proposed) methods.}
\end{figure*}

%

\begin{figure*}[h!]
\begin{minipage}{.02\linewidth}
\rotatebox{90}{\small \textit{ST  \; \; \; \; \;\;\;\;\;AI \;\;\;\;\;\;\; K-means\;\;\;\;\;\;\;\;Init.}}
\end{minipage}
\foreach \c in {3,...,12}{
    \begin{minipage}{0.085\linewidth}
    \includegraphics[width=\linewidth]{images/best_facepos_centroid_init\c.png}\\
    \includegraphics[width=\linewidth]{images/kmean_facepos_centroid_final\c.png}\\
    \includegraphics[width=\linewidth]{images/aff_facepos_centroid_final\c.png}\\
    \includegraphics[width=\linewidth]{images/best_facepos_centroid_final\c.png}
    \end{minipage}
}
\caption{ \textit{Additional} - \textbf{$10$ out of $13$ Face Position Centroids Visualization (dim  288x384)} - Depiction of the initialization of the per-cluster centroids \textbf{top row} and the final per-cluster centroids of the K-means, AI K-means, and ST K-means (proposed) methods.}
\end{figure*}

\end{document}